\documentclass[a4paper,11pt]{article}

\usepackage{amsmath,amsthm,amssymb,amsfonts}
\usepackage{bbm}
\usepackage{dsfont}
\usepackage[ruled,vlined]{algorithm2e}
\usepackage[colorlinks,citecolor=blue]{hyperref}
\usepackage{nicefrac}
\usepackage{xcolor}
\usepackage{graphicx}
\usepackage{caption}
\usepackage{subcaption}
\usepackage{ragged2e}
\usepackage{setspace}
\usepackage{enumitem}
\usepackage{tabularx}
\usepackage{booktabs}

\makeatletter
\newcommand{\shorteq}{%
  \settowidth{\@tempdima}{-}
  \resizebox{10pt}{\height}{ = }%
}
\makeatother

\DeclareMathSymbol{\shortminus}{\mathbin}{AMSa}{"39}

\SetKwInput{KwInput}{Input}
\SetKwInput{KwOutput}{Output}

\newtheorem{theorem}{Theorem}[section]
\newtheorem{corollary}{Corollary}[section]
\newtheorem{lemma}{Lemma}[section]

\newtheorem{proposition}{Proposition}[section]
\theoremstyle{definition}
\newtheorem{definition}{Definition}[section]

\theoremstyle{remark}

\newboolean{showcomments}
\setboolean{showcomments}{true}
\ifthenelse{\boolean{showcomments}}
{ \newcommand{\mynote}[3]{
		\fbox{\bfseries\sffamily\scriptsize#1}
		{\small$\blacktriangleright$\textsf{\emph{\color{#3}{#2}}}$\blacktriangleleft$}}
	\newcommand{\zzz}[1]{{\setlength{\fboxsep}{2pt}\fcolorbox{black}{yellow}{\textsf{\emph{#1}}}}\xspace}}
{ \newcommand{\mynote}[3]{}
	\newcommand{\zzz}[1]{}}

\definecolor{airforceblue}{rgb}{0.36, 0.54, 0.66}
\definecolor{britishracinggreen}{rgb}{0.0, 0.26, 0.15}

\usepackage[round]{natbib}

\title{Fairness Meets Privacy: Integrating Differential Privacy and Demographic Parity in Multi-class Classification}

\author{ Lilian Say$^{(1)}$, Christophe Denis$^{(2)}$, Rafael Pinot$^{(1)}$ }

\begin{document}

\date{}
\maketitle
\begin{center}
$(1)$ Sorbonne Université, LPSM, UMR 8001\\
${(2)}$ Université Paris 1 Panthéon-Sorbonne, SAMM
\end{center}

\begin{abstract}
The increasing use of machine learning in sensitive applications demands algorithms that simultaneously preserve data privacy and ensure fairness across potentially sensitive sub-populations. While privacy and fairness have each been extensively studied, their joint treatment remains poorly understood. Existing research often frames them as conflicting objectives, with multiple studies suggesting that strong privacy notions such as differential privacy inevitably compromise fairness. In this work, we challenge that perspective by showing that differential privacy can be integrated into a fairness-enhancing pipeline with minimal impact on fairness guarantees. We design a post-processing algorithm, called \textsc{DP2DP}, that enforces both demographic parity and differential privacy. Our analysis reveals that our algorithm converges towards its demographic parity objective at essentially the same rate (up logarithmic factor) as the best non-private methods from the literature. Experiments on both synthetic and real datasets confirm our theoretical results, showing that the proposed algorithm achieves state-of-the-art accuracy/fairness/privacy trade-offs.
    
\end{abstract}

{\bf Keywords.} {Multi-class classification, Fairness, Privacy, Demographic parity.}

\vspace{-5pt}
\section{Introduction}
\label{sec:intro}


 Machine learning is nowadays at the core of many data-oriented high-stakes applications such as medicine, finance, heavy industry, or recommendation systems on the internet. However, many studies over the last decade have identified major shortcomings of machine learning models, especially in terms of {\it model bias}~\citep{fairnessAngwin} and {\it violation of data privacy}~\citep{shokri2017membership}. These issues raise questions about the legal liability of model providers and lead practitioners to reevaluate the trust they place in such systems. They call for algorithms that simultaneously meet strong privacy and fairness requirements.
 
\paragraph*{Fairness.} On the one hand, learning algorithms often inherit bias from training data, leading to undesirable effects on future decisions. In particular, algorithms optimized solely for accuracy on historical data can conflict with modern ethical principles, as they tend to perpetuate existing societal biases. To address this, {\it algorithmic fairness} has emerged as a framework for mitigating bias and limiting the influence of sensitive attributes (e.g., gender or ethnicity) on model predictions~\citep{Calders_2009, Hardt_2016, Calmon_2017, Agarwal_2018, Chiappa_2020, Yang_2023}.  Over the past few years, several notions of fairness have been investigated in the classification setting, including \emph{demographic parity}~\citep{Calders_2009}, \emph{equalized odds}, and \emph{equal opportunity}~\citep{Hardt_2016}. Similarly, various algorithmic solution have been proposed to enforce one (or more) of them. These methods are typically grouped into three categories: {\it pre}-processing~\citep{Calmon_2017}, {\it in}-processing~\citep{Agarwal_2018}, and {\it post}-processing~\citep{Hardt_2016}. The effectiveness of each approach depends not only on the chosen fairness criterion but also on the application context, and all methods inevitably involve a trade-off between fairness and model accuracy.

\paragraph*{Privacy.} Besides, protecting the confidentiality of training data has also become a major concern in machine learning, particularly as models are increasingly trained on sensitive information. Accordingly, several notions have been introduced to formalize privacy guarantees, among which {\it differential privacy} has emerged as the dominant standard~\citep{dwork2006calibrating, dwork2014algo}. In practice, differential privacy can be implemented through different strategies. Early approaches perturbed either the output of the learning algorithm or its objective function~\citep{chaudhuri2011dperm}, while more recent methods rely on gradient perturbation during training~\citep{abadi2016deep, altschuler2023}. These techniques provide strong theoretical guarantees, but often degrade model performance, since injecting noise inevitably reduces accuracy. This persistent trade-off underscores the challenge of designing algorithms that are simultaneously private and accurate.

\paragraph*{Integrating Fairness \& Privacy.} Over the past decade, both the challenges of reducing model bias and protecting data privacy in machine learning have been extensively studied. However, integrating these objectives into a single framework and designing algorithms that are simultaneously fair, private, and accurate remains poorly understood, as recently surveyed in~\citep{ijcai2022p766,ferry2023soktamingtriangle}. Most existing work focused on empirically observing that privacy and fairness can sometimes conflict~\citep{bagdasaryan2019dpimpactacc, 10.1145/3411501.3419419, Uniyal2021DPSGDVP, 9581219}. These observations, however, lacked theoretical justification, leaving open questions about why and when such conflicts occur. Along similar lines, some theoretical studies have suggested that fairness and privacy can be incompatible, showing that enforcing one constraint may prevent the other from being satisfied or lead to trivial accuracy~\citep{cummings2019compa, sanyal2022how}. Yet, these results often rely on unrealistic assumptions that seldom hold in practice. In short, although much of the literature suggests that differential privacy and algorithmic fairness may be at odds, a detailed analysis of how to integrate them remains lacking.

\paragraph*{Our contributions.} In this work, we challenge viewpoint that differential privacy and algorithmic fairness are inherently at odds. We show that privacy can be integrated into a fairness-enhancing pipeline with minimal impact on fairness guarantees. Building on recent advances in fair multi-class classification~\citep{Denis_Elie_Hebiri_Hu_2024}, we design a post-processing algorithm, called \textsc{DP2DP}, that enforces both demographic parity and differential privacy under very mild technical assumptions on the problem setting. Our analysis shows that the algorithm converges to its demographic parity objective at a rate of $\mathcal{O}\left( \nicefrac{\log(N)}{\sqrt{N}}\right)$, where $N$ is the number of points used in the post-processing procedure. This rate matches, up to a logarithmic factor, the best non-private methods in the literature. Experiments on both synthetic and real datasets confirm our theoretical results, demonstrating that the proposed algorithm achieves state-of-the-art trade-offs between accuracy, fairness, and privacy. We emphasize that our approach is not a mere adaptation of existing results. In~\citep{Denis_Elie_Hebiri_Hu_2024}, the post-processing step involves minimizing a non-smooth objective, and while a smooth surrogate is used for practical purposes, no theoretical guarantees are provided for it. In contrast, to integrate differential privacy, we adopt a smooth surrogate (specifically, a parametrized log-sum-exp) and our fairness guarantees are established for this proxy. Consequently, our results extend existing results in a non-trivial way, demonstrating that approximate demographic parity can be achieved under both smoothing and privacy constraints.






\paragraph*{Closely Related Work.} Very few prior works have theoretically addressed the integration of fairness and privacy. \cite{jagielski2019dpfl} was the first to propose a post-processing algorithm combining these notions in binary classification, targeting the Equalized Odds fairness goal. However, their scheme only guarantees differential privacy w.r.t. the protected attributes, and its extension to the multi-class setting is unclear. Later, \cite{lowy2023stochastic} introduced an in-processing approach to merge fairness and differential privacy, but as in \citep{jagielski2019dpfl}, only the sensitive attribute is privatized, and no theoretical fairness guarantees are provided 
More recently, \cite{mangold2023dfimpactfairness} presented one of the most promising theoretical results for multi-class classification, showing that differential privacy can have a bounded impact on algorithmic fairness. Yet, this work still lacks explicit fairness guarantees, since it only bounds the fairness loss of the empirical risk minimizer, which has no inherent fairness properties. Finally, \cite{tran2025fairdp} proposed very recently a novel scheme that combines differential privacy with a new fairness criterion based on bounding the distribution of features across sensitive groups. However this approach does not provide a closed-form analysis of privacy, and its fairness notion is hard to connect to established fairness metrics



\vspace{-5pt}
\section{Preliminary on Fairness \& Privacy}\label{sec:background}

     This section is devoted to the introduction of the general framework, as well as  the definition of  approximate demographic parity constraint, and differential privacy. 
    Specifically, we consider the fair multi-class classification setting where every observation is in the form of a random tuple $(X,S,Y)$ with unknown distribution $\mathbb{P}$. The input random features $(X,S)$ respectively take their values in $\mathbb{R}^d$ and $\mathcal{S}$, where $\mathcal{S}=\{-1,1\}$ is called the set of sensitive attributes. The output $Y$ belongs to $[K]$, with $K$ the number of classes. Within this framework, a classifier $g$ is a function that maps $\mathbb{R}^d \times \mathcal{S}$ onto $[K]$, and its accuracy is evaluated through the misclassification risk $R(g) = \mathbb{P}\left(g(X) \neq Y\right)$.

    \paragraph*{Demographic parity constraint.} 
         As discussed in Section~\ref{sec:intro}, fairness can be formalized in many ways. In this paper, we focus on the notion of demographic parity first introduced in~\citep{dwork2012fairness}. This definition, which captures a form of group fairness, is designed around a notion of unfairness that measures the disparity in terms of prediction between the sensitive attributes. More formally, the unfairness of a classifier $g$, denoted $\mathcal{U}(g)$, is defined as 
        \begin{equation}\label{eq:SP}
            \mathcal{U}(g) \shorteq \max_{k \in [K]}\Big|\mathbb{P}_{1}\left[g(X,S) \shorteq k\right] \shortminus \mathbb{P}_{\shortminus 1}\left[g(X,S) \shorteq k\right] \Big|, 
        \end{equation} 
        where for any $s \in \mathcal{S}$, $\mathbb{P}_{s}$ denote the probability conditioned to $S=s$. Using the above, we can define the notion of $\rho$-demographic parity as follows.

        \begin{definition} Let $\rho \geq 0$ and $\mathcal{U}$ as defined in~(\ref{eq:SP}). A classifier $g$ is said to satisfy $\rho$-demographic parity if its unfairness is bounded above by $\rho$, i.e., $\mathcal{U}(g) \leq \rho.$
        \end{definition}

        To enforce demographic parity, we move from simply minimizing the misclassification risk to solving the constrained problem $\min \{ R(g) \mid \mathcal{U}(g) \leq \rho \}.$ As in standard classification, we assume access to labeled observations sampled i.i.d. from $\mathbb{P}$. However, a distinctive advantage of demographic parity is that its unfairness measure depends only on the features and sensitive attributes, not on the labels. Accordingly, unlike other fairness notions, this makes it possible to check or enforce demographic parity without labeled data, see, e.g.,~\citep{Chzhen_2020}. Hence, we assume access to a training dataset of $n$ labeled examples $D_n = \{ (x_i, s_i, y_i) \mid i \in [n] \}$ sampled i.i.d. from $\mathbb{P}$, together with an additional set of $N$ unlabeled examples $D_N = \{ (\tilde{x}_j, \tilde{s}_j) \mid j \in [N]\}$ sampled i.i.d. from $\mathbb{P}_{\mathcal{X}\times\mathcal{S}}$, the marginal of $\mathbb{P}$ on $\mathcal{X}\times\mathcal{S}$. The labeled data is used to train the model and optimize accuracy, while the unlabeled data serves to enforce demographic parity.

        \paragraph*{Differential Privacy. } Privacy in machine learning is most commonly formalized through the notion of differential privacy~\citep{dwork2006calibrating}. Intuitively, a learning algorithm is differentially private if its output does not change significantly when a single individual in the training set is replaced by another. At the dataset level, two datasets $D$ and $D'$ of the same size are said to be adjacent (denoted $D \approx D'$) if they differ in exactly one element. Based on this notion of adjacency, differential privacy is defined as follows.
        
        
       \begin{definition}
        Let $\varepsilon \geq 0$ and $\delta \in [0,1]$. Consider two arbitrary spaces $\mathcal{Z}$ and $\mathcal{V}$, and let $m \in \mathbb{N}$.  
        A randomized mechanism $\mathcal{M}$, that takes as input a dataset $D \in \mathcal{Z}^m$ and outputs a random variable with values in $\mathcal{V}$, is said to be \emph{$(\varepsilon,\delta)$-differentially private} if, for any two adjacent datasets $D \approx D' \in \mathcal{Z}^m$ and for all measurable subsets $V \subseteq \mathcal{V}$, we have
        \begin{equation*}
             \mathbb{P}[\mathcal{M}(D) \in V] 
             \leq e^{\varepsilon} \, \mathbb{P}[\mathcal{M}(D') \in V] + \delta .
        \end{equation*}
        \end{definition}

       Although widely adopted, $(\varepsilon,\delta)$-differential privacy can be challenging to analyze directly, especially for advanced mechanisms. To overcome this limitation, \emph{Rényi differential privacy} was introduced as a relaxation of differential privacy, based on the Rényi divergence~\citep{rdp}. The Rényi divergence of order $\alpha > 1$ between two probability distributions $P$ and $Q$ with probability densities $p$ and $q$ w.r.t. the Lebesgues measure $\lambda$ is defined as
        \begin{equation*}
            D_{\alpha}(P \,\|\, Q) = \frac{1}{\alpha - 1} \log \int_{\mathcal{X}} p(x)^\alpha q(x)^{1-\alpha} d \lambda(x) .
        \end{equation*}

      As $\alpha \to 1$, the Rényi divergence converges towards the Kullback-Leibler divergence. Conversely, as $\alpha \to \infty$, it approaches the maximum divergence, commonly used in the differential privacy literature~\citep{dwork2014algo}. Based on this definition, Rényi differential privacy can then be formalized as follows.

     \begin{definition}\label{def:rdp}
        Let $\alpha \geq 1$ and $\varepsilon \geq 0$. A randomized mechanism $\mathcal{M}$ satisfies $(\alpha, \varepsilon)$-Rényi differential privacy if, for all neighboring datasets $D \approx D'$,
        \begin{equation*}
            D_{\alpha}\left( \, \text{Law}\left(\mathcal{M}(D) \right) \, \Vert \, \text{Law}\left(\mathcal{M}(D')\right) \, \right) \leq \varepsilon, 
        \end{equation*}
         where $\text{Law}(\mathcal{M}(D))$ and $\text{Law}(\mathcal{M}(D'))$ are the reference probability distributions of $\mathcal{M}(D)$ and $\mathcal{M}(D')$. 
        \end{definition}

      When designing a differentially private mechanism to estimate a function $f : \mathcal{Z}^n \to \mathbb{R}^p$, it is essential to quantify how much the output of $f$ can change when a single data point in the dataset is modified. This sensitivity is captured by the $\ell_2$ norm of the difference between outputs on adjacent datasets, i.e.,  
    \begin{equation}
    \label{eq:sensitivity}
         \Delta(f) = \max_{D \approx D'} \, \| f(D) - f(D') \|_2,
    \end{equation}
    Once the sensitivity is known, $f(D)$ can be released privately using the Gaussian mechanism, which adds Gaussian noise to any output of $f$. More precisely, for any $D \in \mathcal{Z}^n$, the mechanism is defined as  
    $\mathcal{M}(D) = f(D) + Z$ with $Z \sim \mathcal{N}(0, \sigma^2 I_p)$, and with $I_p$ denoting the identity matrix of dimension $p$. Then, it can be shown that $\mathcal{M}$ satisfies $(\alpha, \tfrac{\alpha \Delta(f)^2}{2\sigma^2})$-Rényi differential privacy for any $\alpha \geq 1$~\citep{rdp}. The Gaussian mechanism is a fundamental building block of privacy-preserving algorithms, and its usefulness extends well beyond one-shot releases. In practice, it is often combined with more advanced techniques. For example, together with the privacy amplification by subsampling phenomenon~\cite{balle2018subsamp}, it enables the design of iterative private learning algorithms such as differentially private stochastic gradient descent, a.k.a., DP-SGD~\citep{abadi2016deep,altschuler2023}.
    In this context, the Rényi differential privacy framework plays a crucial role, as it allows for tight privacy accounting in iterative settings like DP-SGD. Moreover, guarantees expressed under this definition be seamlessly converted back into standard $(\varepsilon,\delta)$-differential privacy guarantees, making it one of the gold-standard approaches for analyzing privacy in modern training pipelines.



\section{DP2DP: Integrating Fairness \& Privacy at Minimal Cost}\label{sec:method}

    We propose a two-phase training pipeline that integrates demographic parity with differential privacy. In the first phase, we use the labeled training set to build a private classifier. At this stage, the classifier provides privacy guarantees w.r.t. the labeled data $D_n$ but does not yet satisfy the fairness. In the second phase, we introduce a privacy-preserving post-processing step, called \textsc{DP2DP}, which exploits the unlabeled data to enforce demographic parity while maintaining differential privacy w.r.t. the unlabeled data $D_N$. 

    

    \paragraph*{Phase 1.}  The first phase consists of obtaining a differentially private probit-based classifier that estimates the conditional probabilities 
    \[
    \mathbb{P}\left [ Y=k \mid (X,S)=(x,s) \right], 
    \]  
    for every $k \in [K]$ and $(x,s) \in \mathcal{X} \times \mathcal{S}$.
    A natural approach to do so is to train the model from scratch on the labeled dataset $D_n$ using a state-of-the-art differentially private methods such as DP-SGD. This ensures that the estimated probabilities, denoted by $\left(\bar{p}_k(x,s)\right)_{k \in [K]}$, are computed in a privacy-preserving manner. Since no modifications are made to the training algorithm at this stage, the privacy guarantees follow directly from existing DP-SGD accounting methods~\citep{altschuler2023}.  

    Alternatively, if the training process cannot be modified, or if an already trained model is available, differential privacy can still be enforced through output perturbation~\citep{chaudhuri2011dperm}. In this case, a typical strategy is to apply the Gaussian mechanism to the estimated probabilities at the end of training. As explained in Section~\ref{sec:background}, the strength of the resulting privacy guarantee depends on the sensitivity of the computed probabilities to the modification of a single point in the dataset. Without additional information about the probit-based model, one may need to adopt a pessimistic estimate of the sensitivity, which limits the achievable privacy guarantees.  However, if additional structural assumptions are available, tighter guarantees can be obtained. For instance, when the classifier is linear and the loss function is strongly convex and Lipschitz in the parameters, the sensitivity of the estimated probabilities is of order $\mathcal{O}(1/n)$~\citep{chaudhuri2011dperm}. Consequently, the amount of noise required by the Gaussian mechanism diminishes as $n$ increases, leading to tighter privacy guarantees. 



    In both cases, the outcome of Phase 1 is a set  of differentially private conditional probability estimators $(\bar{p}_k)_{k \in [K]}$ w.r.t. $D_n$, which serve as inputs for our fairness enhancing mechanism in Phase 2.



    \paragraph*{Phase 2.} We now describe our post-processing procedure, called \textsc{DP2DP}, which leverages the unlabeled dataset $D_N$ to enforce demographic parity while preserving the privacy guarantees established in Phase~1. Our objective is to recalibrate the fairness of the probit-based classifier so that it satisfies $\rho$-demographic parity via Lagrangian regularization. Concretely, we build a classifier of the form
\begin{equation*}
    \hat{g}_{\rho}(x, s) = \arg\max_{k \in [K]} \bar{\pi}_s \bar{p}_k(x, s) - s(\bar{\lambda}^{(1)}_k - \bar{\lambda}^{(2)}_k),
\end{equation*}
    where $\bar{\pi}_s$ denotes a privacy-preserving estimate of the frequency of each sensitive attribute in the unlabeled dataset $D_N$, and the Lagrange multipliers $(\bar{\lambda}^{(1)}, \bar{\lambda}^{(2)}) \in [0,C_\lambda]^{2K}$, for a fixed parameter $C_{\lambda} \in \mathbb{R}_+^*$, are obtained as the (privatized) solution to the Lagrangian relaxation of the empirical risk minimization problem under the $\rho$-demographic parity constraint. At a high level, \textsc{DP2DP} operates in two steps: (i) it privatizes the group proportions, enabling privacy-preserving estimation of fairness constraints, and (ii) it optimizes the corresponding Lagrange multipliers through a smoothed and privatized objective. 


    \textbf{i)} First, observe that the empirical frequencies, defined as $\hat{\pi}_s = \tfrac{1}{N}\sum_{i=1}^N \mathds{1}(\tilde{s}_i = s),$ for all $s \in \mathcal{S}$ depend directly on the unlabeled dataset $D_N$. To prevent any information leakage, we privatize these quantities using the Gaussian mechanism, i.e., we set 
    \begin{equation}\label{eq:pi}
      \bar{\pi}_s = \hat{\pi}_s + Z, \text{ with } Z \sim \mathcal{N}(0, \sigma_\pi^2), 
    \end{equation}
    where $\sigma_\pi^2$ is the amount of noise used to privatize the probabilities. Then, all subsequent computations rely exclusively on the privatized proportions $\bar{\pi}_s$.

    \textbf{ii)} After obtaining the privatized proportions $\bar{\pi}_s$, the next step is to compute the Lagrange multipliers $(\bar{\lambda}^{(1)}, \bar{\lambda}^{(2)})$ by solving an empirical contrast problem. Carrying out this optimization in a privacy-preserving way is non-trivial, mainly because the objective is non-smooth: it involves a sum of $\max$ operators, which prevent the direct use of gradient-based methods. To address this difficulty, we adopt the smoothing technique of~\cite{nesterov2005smooth}, where non-differentiable convex functions are approximated by smooth surrogates defined via so called strongly convex \emph{prox-functions}. In particular, choosing the negative entropy $\Omega(z) = \sum_{k=1}^K z_k \log(z_k)$ as prox-function yields the well-known log-sum-exp ($\mathrm{LSE}_\beta$) smoothing, for $\beta > 0$, defined as
    \begin{equation*}
    \mathrm{LSE}_\beta(z) = \beta \log \left(\sum_{k=1}^K \exp(z_k/\beta)\right), \forall z \in \mathbb{R}^K.
    \end{equation*}
    This approximation is particularly useful because the negative entropy is $1$-strongly convex w.r.t. $\|\cdot\|_1$, which implies that $\mathrm{LSE}_\beta$ is $1/\beta$-smooth w.r.t. $\|\cdot\|_\infty$ \citep{niculae_regularized_2017}. Moreover, its gradient admits a simple closed form see Section~\ref{sec:guarantees} for more details. Plugging this smoothing into the empirical risk minimization yields the surrogate objective  
    \begin{equation}\label{eq:H-beta}
        \hat{H}_\beta(\lambda^{(1)}, \lambda^{(2)}) 
        = \frac{1}{|\mathcal{S}|} \sum_{s \in \mathcal{S}} \frac{1}{N_s} \sum_{x \in D_{N \mid \mathcal{X}, s}} \hat{h}(\lambda; x ,s),
    \end{equation}
    where for each $s \in \mathcal{S}$, $D_{N \mid \mathcal{X}, s} = \{x \in \mathcal{X} \mid (x,s) \in D_N \}$ is the set of non-sensitive features of $D_N$ whose sensitive attribute is $s$, $N_s$ is the size of $D_{N \mid \mathcal{X}, s}$. Furthermore, the per-sample loss $\hat{h}$ is defined, for any $(x,s) \in \mathcal{X} \times \mathcal{S}$ and $\lambda \in [0,C_{\lambda}]^{2K}$, as  
    \begin{equation}\label{eq:hhat}
        \hat{h}(\lambda; x,s) \shorteq |\mathcal{S}| \, \mathrm{LSE}_{\beta}\!\left(\ell^s(x; \lambda)\right) 
        + \rho \sum_{k=1}^K \lambda^{(1)}_k + \lambda^{(2)}_k,
    \end{equation}
    In the above $\ell^s(x;\lambda)$ simply abbreviates the vector filled with corrected probabilities estimations, i.e., $\ell^s(x;\lambda)= ( \bar{\pi}_s \bar{p}_k(x, s) - s(\lambda^{(1)}_k - \lambda^{(2)}_k) )_{k \in [K]}$. Using this surrogate objective $\hat{H}$, the multipliers are computed through a differentially private stochastic gradient scheme. At each iteration, a mini-batch $\{(\bar{x}_1, \bar{s}_1), \dots, (\bar{x}_b, \bar{s}_b)\}$ of size $b \in [n]$ is drawn from $D_N$. Each element of this mini-batch $(\bar{x}_j, \bar{s}_j)$ is sampled by first sampling a group $ \bar{s}_j \in \mathcal{S}$ uniformly at random, and then sampling an example $\bar{x}_j$ from $D_{N \mid \mathcal{X},  \bar{s}_j}$, i.e., the set of non-sensitive features of $D_N$ whose sensitive attribute is $\bar{s}_j$. The stochastic gradient of the smoothed loss is averaged over the minibatch and perturbed with Gaussian noise, yielding a noisy update direction $\bar{u}_t$. The update rule is then
    \begin{equation*}
    \lambda^{t+1} \gets \Pi_{[0,C_\lambda]^{2K}} \big(\lambda^t - \eta_t \bar{u}_t\big), 
    \end{equation*}
    where $\eta_t$ is called the learning rate at step $t$, and $\Pi_\mathcal{K}$ is the Euclidean projection operator on a convex set $\mathcal{K}$. After $T$ iterations, the algorithm outputs the final private and fair classifier, using $\lambda^T$ as regularization. The entire post-processing method, called \textsc{DP2DP}, is summarized in Algorithm~\ref{algo:dp-fair} below. To improve readability, we color-code privacy related steps in \textcolor{purple}{purple} and fairness related computations in \textcolor{orange}{orange}. 
    

     \begin{algorithm}[ht]
        \caption{DP2DP algorithm}
        \label{algo:dp-fair}
        \setstretch{1.15}
        \begin{justify}
        \setstretch{1}
        \medskip
        \textbf{Input:} Let $D_N = \{(\tilde{x}_1, \tilde{s}_1), \ldots, (\tilde{x}_N, \tilde{s}_N)\}$ be the unlabeled dataset, and $\bar{p}_1, \dots, \bar{p}_K $ the probability estimators designed during Phase 1. Furthermore, let ${\color{orange} \rho \geq 0}$ be the target unfairness level, $\textcolor{orange}{\hat{h}}$ the associated per-sample loss as defined in~(\ref{eq:hhat}), and ${\color{purple} (\sigma_\pi, \sigma_{\textsc{sgd}})}$ be the privacy noises. Finally, let $\beta$ be the smoothing constant, $T$ be the total number of iterations,  $(\eta_t)_{t \in [T]}$ be the sequence of learning rates, $b \in [N]$ the batch size per iteration, and $C_\lambda >0$ the projection threshold.
        \end{justify}
        \textbf{i) Compute empirical frequencies}\\
        ~\For{$s \in \mathcal{S}$} {
            $\bar{\pi}_s \gets \frac{1}{N} \sum_{i =1}^N \mathds{1}\left(\tilde{s}_i=s \right) \textcolor{purple}{+ Z}$, \\
            with $\textcolor{purple}{Z \sim \mathcal{N}\left(0, \sigma_\pi^2\right) }$. 
        } \vspace{5pt}
        
        \textbf{ii) Compute Lagrange parameters}\\
        ~\For{$t \in [T]$}{
            Sample minibatch $\{(\bar{x}_j, \bar{s}_j)\}_{j=1}^{b}$ where $\bar{s}_j \sim \mathrm{Unif}(\mathcal{S})$ and $\bar{x}_j \sim \mathrm{Unif}( D_{N \mid \mathcal{X}, \bar{s}_j})$\\ \label{step:sample}
            \textbf{Compute gradients}\\
            For each $j \in [b]$, compute the update direction
            ${\color{orange} u_t(\bar{x}_j, \bar{s}_j) \gets \nabla_{\lambda^t}\hat{h}(\lambda^t; \bar{x}_j, \bar{s}_j)}$\\
            \textbf{Add noise}\\
            $\bar{u}_t \gets \frac{1}{b}\left(\sum_{j=1}^{b} u_t(\bar{x}_j, \bar{s}_j) \textcolor{purple}{ +  Z_t}\right)$, with $\textcolor{purple}{Z_t \sim \mathcal{N}(0, \sigma_{\textsc{sgd}}^2 I_{2K})}$\\
            \textbf{Descent}\\
            ${\color{orange} \lambda^{t+1} \gets \Pi_{[0, C_\lambda]^{2K}} \left(\lambda^t - \eta_t \bar{u}_t\right) }$
        }
        \medskip
        ~$\bar{\lambda}^{(1)} \gets \left(\lambda^{T}_1, \ldots, \lambda^{T}_K \right)$ ~ ; ~$\bar{\lambda}^{(2)} \gets \left(\lambda^{T}_{K+1}, \ldots, \lambda^{T}_{2K} \right)$ \medskip
        
        \KwOutput{
            $\hat{g}_\rho \gets \underset{k \in [K]}{\arg\max}\left\{\bar{\pi}_s \bar{p}_k - s\left(\bar{\lambda}^{(1)}_k - \bar{\lambda}^{(2)}_k\right)\right\}$
        }
        
    \end{algorithm}
\vspace{-5pt}

    

\section{Privacy guarantees}\label{sec:privacy}
    
We begin by analyzing the privacy cost of our method. The privacy guarantees of the probit computations in Phase~1 depend on the differentially private training procedure employed, and therefore directly follow from either~\cite{altschuler2023} or~\cite{chaudhuri2011dperm}. Importantly, Phases~1 and~2 are independent from a privacy perspective: Phase~1 only accesses samples from $D_n$, while Phase~2 can be viewed as a post-processing step that operates solely on $D_N$. Consequently, if Phase~1 satisfies $(\epsilon_1, \delta_1)$-differential privacy, then to obtain a global privacy guarantee it suffices to show that Phase~2 satisfies $(\epsilon_2, \delta_2)$-differential privacy. By parallel composition, the entire pipeline then enjoys $(\max\{\epsilon_1, \epsilon_2\}, \max\{\delta_1, \delta_2\})$-differential privacy. In this section, we establish the privacy guarantees of Phase~2 (i.e., for the \textsc{DP2DP} algorithm), and leave the complete composition to Section~\ref{sec:expe}.

\paragraph*{Privacy for DP2DP.} The privacy guarantees of \textsc{DP2DP} stem from two independent sources of randomness: i) the privatization of the group frequencies $(\hat{\pi}_s)_{s \in \mathcal{S}}$, and ii) the Gaussian perturbations introduced in the DP-SGD updates when optimizing the smoothed objective $\hat{H}_\beta$. These two components jointly determine the overall privacy cost of Phase~2. In Theorem~\ref{theorem:rdp-algo}, we establish a Rényi privacy bound that accounts for both contributions and thereby characterizes the privacy guarantees of \textsc{DP2DP}. The full proof of this theorem can be found in the Appendix.

        \begin{theorem}\label{theorem:rdp-algo} 
        Consider the {\rm DP2DP} scheme, as in Algorithm~\ref{algo:dp-fair}. If the step-size sequence is such that $\eta_t = \eta \leq 2\beta$ for all $t \in [T]$, then {\rm DP2DP} satisfies $(\alpha, \varepsilon)$-Rényi differential privacy for all $\alpha \geq 1$, where
            \begin{equation*}
                \varepsilon \leq \frac{\alpha}{2 N^2 \sigma_\pi^2} + \Psi\left(T, b ,N, \eta, \sigma_{{\rm SGD}} \right)
            \end{equation*}
            where $\Psi := \Psi\left(T, b ,N, \eta, \sigma_{\rm{SGD}} \right) $ is defined as
            \[ \Psi = { \min\left\{T Q , \min_{\underset{\sigma_1^2 + \sigma_2^2 = \sigma_{\rm{sgd}}^2}{\sigma_1, \sigma_2 > 0} } \min_{M \in [T-1]} M Q + \frac{\alpha 2K C_\lambda^2}{2 \eta^2 \sigma_1^2 M } \right\} }, \]
            $Q = S_\alpha\!\left(\frac{b}{N}, \frac{b \sigma_2}{4}\right)$, and for any $(q, \sigma) \in [0,1]\times \mathbb{R}_+$ we define $S_\alpha\!\left(q, \sigma \right)$ as the Rényi divergence of level $\alpha$ between a Gaussian distribution $\mathcal{N}(0, \sigma^2)$ and a mixture of Gaussian $(1-q)\mathcal{N}(0, \sigma^2) + q \mathcal{N}(1, \sigma^2)$.
        \end{theorem}
        \begin{proof}[Skecth of Proof]
           The privacy loss of \rm{DP2DP} naturally decomposes into two components. The first arises from the release of the privatized group proportions $(\bar{\pi}_s)_{s \in \mathcal{S}}$. Since modifying a single individual in the dataset changes these proportions by at most $1/N$, the sensitivity is small. By adding Gaussian noise scaled to this sensitivity, we obtain a Rényi differential privacy guarantee for this step. The second source of privacy loss comes from the DP-SGD optimization of the smoothed objective function $\hat{H}$. Here, each per-sample gradient has sensitivity at most $4$, and the resulting privacy bound follows from including this sensitivity computation in \cite[Theorem~3.1]{altschuler2023}. To conclude, the two contributions are combined using standard Rényi differential privacy version of the composition rule, yielding the overall privacy bound stated in Theorem~\ref{theorem:rdp-algo}.
        \end{proof}


        


        \paragraph*{Discussion.}
            Theorem~\ref{theorem:rdp-algo} provides a general Rényi differential privacy guarantee, offering a relatively tight quantification of the privacy loss. Such a bound is particularly useful for precise privacy accounting in practice. However, it is often more convenient to express guarantees in the standard $(\epsilon,\delta)$-differential privacy framework. A key implication of Theorem~\ref{theorem:rdp-algo} that, under suitable conditions on the order $\alpha$ of Rényi divergence and on the noise $\sigma_{\rm{sgd}}$, the guarantee can be specialized to yield $(\epsilon,\delta)$-differentially private guarantees with explicit noise calibrations. In particular, setting noise parameters as follows
            \begin{equation*}
                \sigma_\pi < \sqrt{2} N \left(\sqrt{\log(1/\delta) + \epsilon} - \sqrt{\log(1/\delta)}\right), \text{   and }
            \end{equation*}
            
            \begin{equation*}
                \sigma_{\rm{sgd}}^2 = \frac{16 \min\left\{T, \left\lceil \tfrac{\sqrt{2K} C_\lambda N}{4 \beta}\right\rceil\right\}}{N^2 \left(\sqrt{\log(1/\delta) + \epsilon} - \sqrt{\log(1/\delta)}\right)^2 - \tfrac{1}{2} \sigma_\pi^2}, 
            \end{equation*}
             one recovers $(\epsilon,\delta)$-differential privacy (see the Appendix for details). We stress that this parametrization is provided as an interpretative tool to make Theorem~\ref{theorem:rdp-algo} more transparent. However, the resulting bounds are significantly looser than those obtained directly from the theorem. Therefore, in practice, privacy guarantees should always be computed using the exact Rényi differential privacy results and converted to standard $(\epsilon,\delta)$-DP numerically rather than relying on this simplified analytical form.

    \section{Fairness guarantees}
    \label{sec:guarantees}
    
    Having established the privacy properties of our approach, we now turn to its fairness guarantees. Recall that our objective is not only to preserve individual privacy but also to enforce demographic parity in the resulting classifier. The theorem below formalizes this by providing a bound on the expected unfairness of the classifier produced by our method.
        
        \begin{theorem}\label{thm:fairness}
        Consider {\rm DP2DP} scheme, as defined in Algorithm~\ref{algo:dp-fair} with a fixed stepsize $\eta_t = \eta \leq 2\beta$ for all $t \in [T]$. Let us also denote by $\pi_{\min} = \min\{ \mathbb{P}\left[ S = s \right] \mid s \in \mathcal{S} \}$ the minimum group size within sensitive attributes w.r.t. $\mathbb{P}$. Then, there exist constants $C_1 > 0$ depending on $K$ and $\pi_{\min}$, $\gamma > 0$, and  $C_2 > 0$ that depends on $K$, and $\gamma$, such that for any conditional probabilities $\bar{p}_k$ computed in Phase 1, one has
            \begin{equation*}
 \mathbb{E}\!\left[\mathcal{U}\left(\hat{g}_\rho \right)\right] \leq \rho + \frac{C_1}{\sqrt{N}} + C_2 e^{-\tfrac{\gamma}{\beta}}  + 4\sqrt{\tfrac{ \sqrt{2K} C_\lambda \log T}{\beta \sqrt{T}}\left(2\sqrt{2} + \rho \sqrt{2K} + \tfrac{\sigma_{\rm{sgd}}\sqrt{2K}}{b}\right)}.
            \end{equation*}
            In the above, the expectation is taken over the sampling of $D_N$ and the randomness of the algorithm (mini-batch sampling and Gaussian noise).
        \end{theorem}
        \begin{proof}[Proof sketch]
            The fairness bound can be understood as the combination of three distinct contributions.  First, the baseline term $\rho + \tfrac{C_1}{\sqrt{N}}$ reflects the intrinsic constraint relaxation present even in the non-private fair classifier, and therefore captures the best achievable approximation to demographic parity in this setting. The control of this term essentially sums up to bounding the deviation between a cumulative distribution function and its empirical counterpart. Second, the use of the $\mathrm{LSE}_\beta$ smoothing introduces an additional approximation error. This approximation of the $\max$ operator, induces an exponential term $C_2 e^{-\gamma/\beta}$, vanishes as the smoothing parameter $\beta$ decreases. Finally, the privatization step contributes a further error term. This term, alike existing bounds in private non-fair optimization such as~\cite{bassily2014privateerm}, depends explicitly on the number of iterations $T$, the smoothing parameter $\beta$, and the noise level $\sigma_{{\rm sgd}}$. Carefully combining these components yiedls the result. \end{proof}
        
        This theorem establishes that the unfairness of our method remains close to the target level $\rho$. To make this guarantee more interpretable, Corollary~\ref{cor:fairness} presents a simplified version of Theorem~\ref{thm:fairness} derived under natural parameter choices.

        \begin{corollary}\label{cor:fairness}
            Let $\rho \ge 0$, and fix $T=N^2$ and $\beta = 2\gamma/\log N$. Then there exists a constant $C_*$, depending on $K$ ,$\pi_{\min}$, $C_\lambda$ ,$ b$ ,$\rho$, and $\sigma_{{\rm sgd}}$ such that
            \begin{equation*}
                \mathbb{E}\!\left[\mathcal{U}(\hat{g}_\rho)\right] \leq \rho + C_* \frac{\log N}{\sqrt{N}}.
            \end{equation*}
            As before, the expectation is taken jointly over the sampling of $D_N$ and the randomness of the algorithm, including mini-batch selection and Gaussian noise.
        \end{corollary}

        \paragraph*{Discussion.} The above shows that, under natural parameter choices, the unfairness gap to the target $\rho$ decays at rate $C_* \tfrac{\log(N)}{\sqrt{N}}$. Importantly, the constant $C_*$ depends explicitly on the noise parameter $\sigma_{{\rm sgd}}$, as we have $C_* \in \mathcal{O}(\sqrt{\sigma_{{\rm sgd}}})$. This dependence directly connects the fairness analysis to the privacy calibration in Section~\ref{sec:privacy}. In particular, when the noise level is chosen to ensure $(\varepsilon,\delta)$-differential privacy and $T,\beta$ are fixed appropriately, we obtain $\sigma_{{\rm sgd}} \in \mathcal{O}\!\left(\sqrt{ \log N/ N }\right)$. Substituting this into $C_*$ shows that privatization does not affect the asymptotic convergence rate more than a factor $\log(N)$ compared to the state-of-the-art non-private result from~\cite{Denis_Elie_Hebiri_Hu_2024}.

        Taken together, the results of Section~\ref{sec:privacy} and Section~\ref{sec:guarantees} demonstrate that privacy-preserving mechanisms can be integrated into fairness pipelines with only minimal overhead. In fact, the additional fairness cost induced by privatization vanishes as $N$ grows, ensuring that the algorithm approaches the target level $\rho$ at rate $\tfrac{\log(N)}{\sqrt{N}}$ while remaining differentially private. It is important to note that the bound obtained in Theorem~\ref{thm:fairness} does not depend either on the sampling of the labels dataset $D_n$ nor on the privacy preserving mechanism applied in Phase~1. Notably, our fairness result is distribution-free which is a clear advantage w.r.t. the in-processing algorithm proposed in~\citep{lowy2023stochastic}. Indeed, in-processing methods involve the minimization, over a suitable family of predictors, of an empirical risk. Therefore, theoretical guarantees on the resulting classifier require assumptions on the considered set of predictors. In contrast, we do not demand any assumption on the set $(\bar{p}_k)_{k \in [K]}$.


    
    \section{Experimental Results}\label{sec:expe}
        We now present empirical results supporting our theoretical analysis. We begin with controlled synthetic experiments, which allow us to investigate the line between privacy and fairness across different regimes. We then turn to real-world datasets, where we compare our method against two baselines to demonstrate its effectiveness in practical settings. The predictive performance is measured by empirical accuracy on a held-out test set that we denote $\mathcal{T}$. Fairness is quantified using empirical demographic parity, i.e., 
        \begin{equation*}
            \hat{\mathcal{U}}(g) = \max_{k \in [K]} \left\lvert \nu_{g \mid 1}(k) - \nu_{g \mid -1}(k)\right\rvert,
        \end{equation*}
        where $\nu_{g \mid s}(k)$ is the empirical distribution of the predicted labels conditional on $S=s$, i.e., $g(X,s)=k$. 
        Smaller values of $\hat{\mathcal{U}}(g)$ indicate stronger fairness guarantees. For privacy accounting, across all experiments, we fix $\delta=o(1/N)$ and compute the corresponding $\varepsilon$ bounds given the noises using the \texttt{dp-accounting} library with exact privacy computation \citep{google_dp_accounting}. 
    
        \paragraph*{Bounded impact of privacy on fairness.} The objective of our first experiment is to validate the fact that our scheme does not suffer an important fairness loss compared to a non-private scheme. To do so, we compare our fairness performance against the state-of-the-art in fair multiclass classification from ~\citep{Denis_Elie_Hebiri_Hu_2024} across several unfairness regimes. We follow the data generation procedure of this paper which allows explicit control over the degree of unfairness via a parameter $p \in [0,1]$. Setting $p=0.5$ produces fair data, while $p \in \{0,1\}$ yields maximally unfair distributions. Data are sampled from a Gaussian Mixture Model with 10 components, feature dimension $d=20$, and 6 classes. We generate $10{,}000$ samples split into 60\% training, 20\% unlabeled, and 20\% test subsets, and vary the unfairness parameter across $p \in \{0.5,0.6,0.7,0.8,0.9,0.99\}$. In line with \cite{Denis_Elie_Hebiri_Hu_2024}, the classifier is trained only on the non-sensitive attribute.  We use a multiclass logistic regression model to estimate conditional class probabilities. Output perturbation is then applied to obtain $\bar{p}_k$ as described in Section~\ref{sec:method}. DP2DP is ran with smoothing parameter $\beta=10^{-5}$, and DP-SGD is used with batch size $b=128$ over $T=100$ iterations, with learning rate $\eta_t = 1/\sqrt{t}$. Results are averaged over $30$ independent runs. For details and additional experiments, refer to Appendix.
            
            \begin{figure}[ht]
                \centering
                \begin{subfigure}{0.48\columnwidth}
                    \centering
                    \includegraphics[height=0.28\textheight, keepaspectratio]{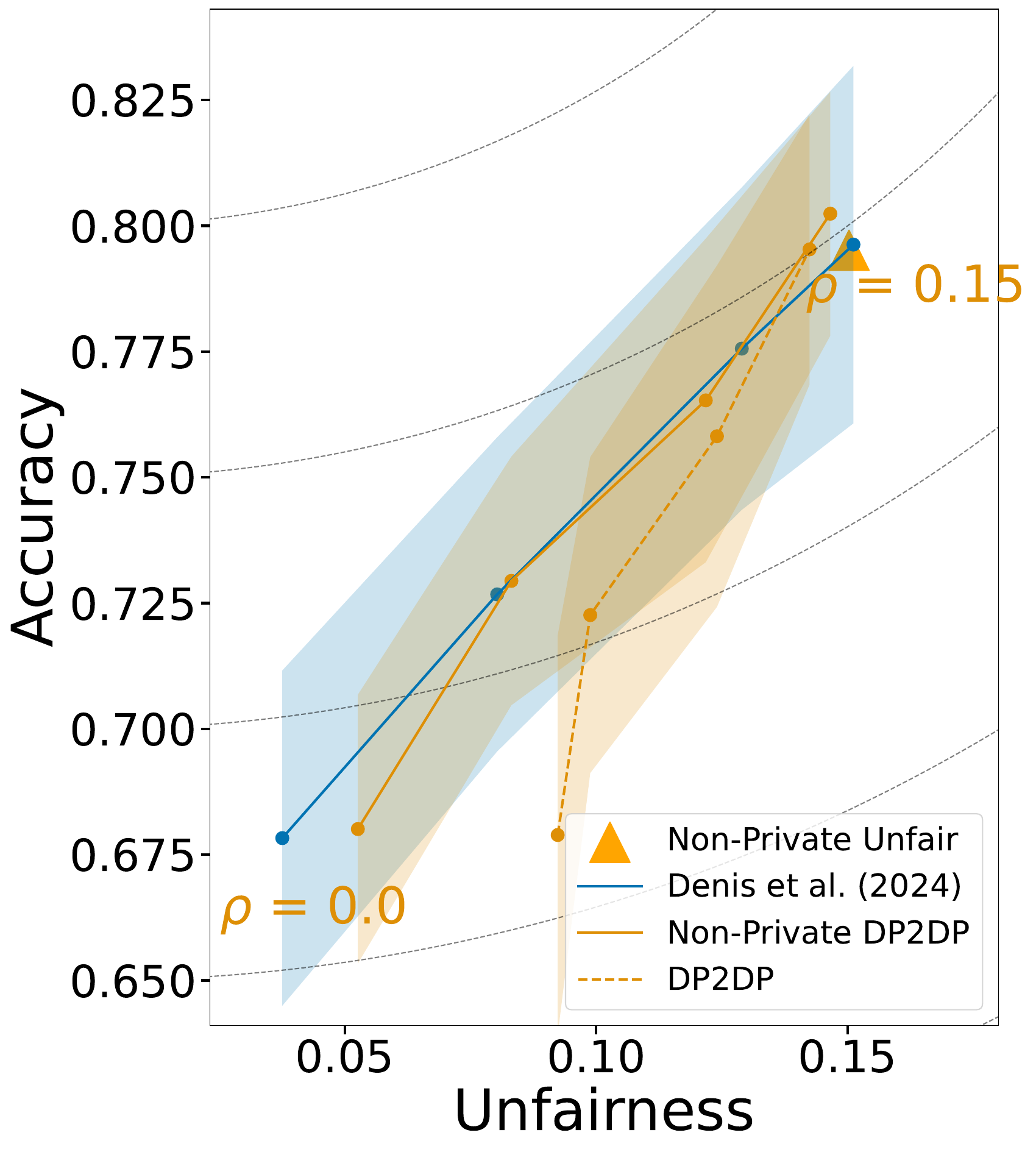}
                    \label{fig:diag_vs_p}
                \end{subfigure}
                \hfill
                \begin{subfigure}{0.48\columnwidth}
                    \centering
                    \includegraphics[height=0.28\textheight, keepaspectratio]{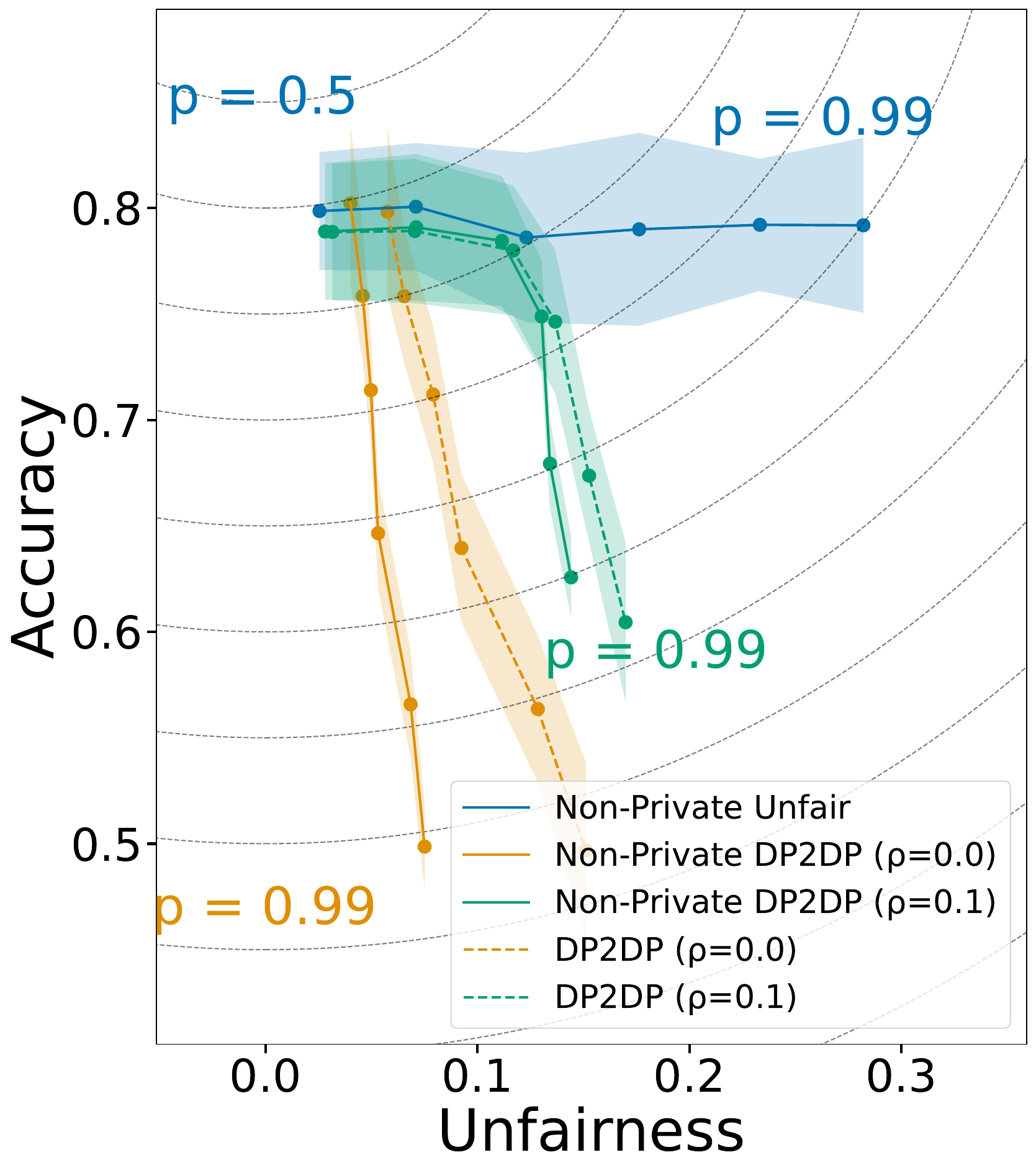}
                    \label{fig:diag_vs_rho}
                \end{subfigure}
            
                \caption{Synthetic experiments with $d=20$ features, $K=6$ classes, $10{,}000$ samples. Left: varying fairness tolerance $\rho$ ($p=0.75$). Right: varying unfairness parameter $p$. The combination of Phase 1 and Phase 2 satisfy $(0.46, 10^{-5})$-differential privacy.}                \label{fig:diagram_synth}
            \end{figure}

            Figure~\ref{fig:diagram_synth} (left) reports the performance of the our method as the fairness tolerance $\rho$ varies. We observe that even with the $\mathrm{LSE}_\beta$ smoothing, our non-private DP2DP essentially matches the fairness/accuracy profile of \cite{Denis_Elie_Hebiri_Hu_2024}, indicating that the approximation error induced by smoothing is negligible in practice. More importantly, the private DP2DP with $\varepsilon=0.46$ remains very close to both non-private baselines. This illustrates the privacy/fairness/accuracy trade-off predicted by Theorem~\ref{thm:fairness}: the privatization step induces only a controlled degradation, while the expected unfairness of the resulting classifier stays within the theoretical guarantees around the target $\rho$.
            
            Figure~\ref{fig:diagram_synth} (right) investigates the effect of the intrinsic unfairness of the data, controlled by parameter $p$. Since our non-private DP2DP already tracks \cite{Denis_Elie_Hebiri_Hu_2024} method closely, we adopt it as the reference baseline. The results show that even in highly biased regimes (e.g., $p=0.99$), our private DP2DP remains close to its non-private counterpart. This empirical observation aligns with Theorem~\ref{thm:fairness} and Corollary~\ref{cor:fairness}, which guarantee that the unfairness gap to $\rho$ remains bounded and converges at rate $\tfrac{\log(N)}{\sqrt{N}}$, despite the additional noise required for privacy.

        \paragraph*{Comparing to state-of-the-art of real data.} The objective of our second experiment is to demonstrate our experimental superiority w.r.t. the state of the art. To do so, we compare our method to 2 baselines attempting to merge privacy and fairness in classification tasks existing private baselines~\citep{tran2021a,tran2021b,lowy2023stochastic} on the fairness-aware benchmark Adult Income dataset \citep{adult_2}. This dataset contains 48,842 records from the U.S. Census, each described by 14 demographic attributes such as age, gender, education, and occupation. We designate gender as the sensitive attribute and focus on the binary classification task of predicting whether an individual earns more than $\$50k$ per year. Following the preprocessing pipeline of \cite{lowy2023stochastic}, categorical and numerical features are transformed into $102$ quantitative variables. The data are split into $55\%$ training, $20\%$ unlabeled, and $25\%$ test sets. We employ a binary logistic regression, applying output perturbation to the class probabilities to obtain $\bar{p}_k$. As in the synthetic setting, DP2DP is run with $\beta=10^{-5}$ and DP-SGD with batch size $b=128$ for $T=100$ iterations, with learning rate $\eta_t = 1/\sqrt{t}$. Results are averaged over 15 independent runs. For details and additional experiments, refer to Appendix.
            
            \begin{figure}
                \begin{subfigure}[b]{0.49\columnwidth}
                    \centering
                    \includegraphics[width=0.8\textwidth]{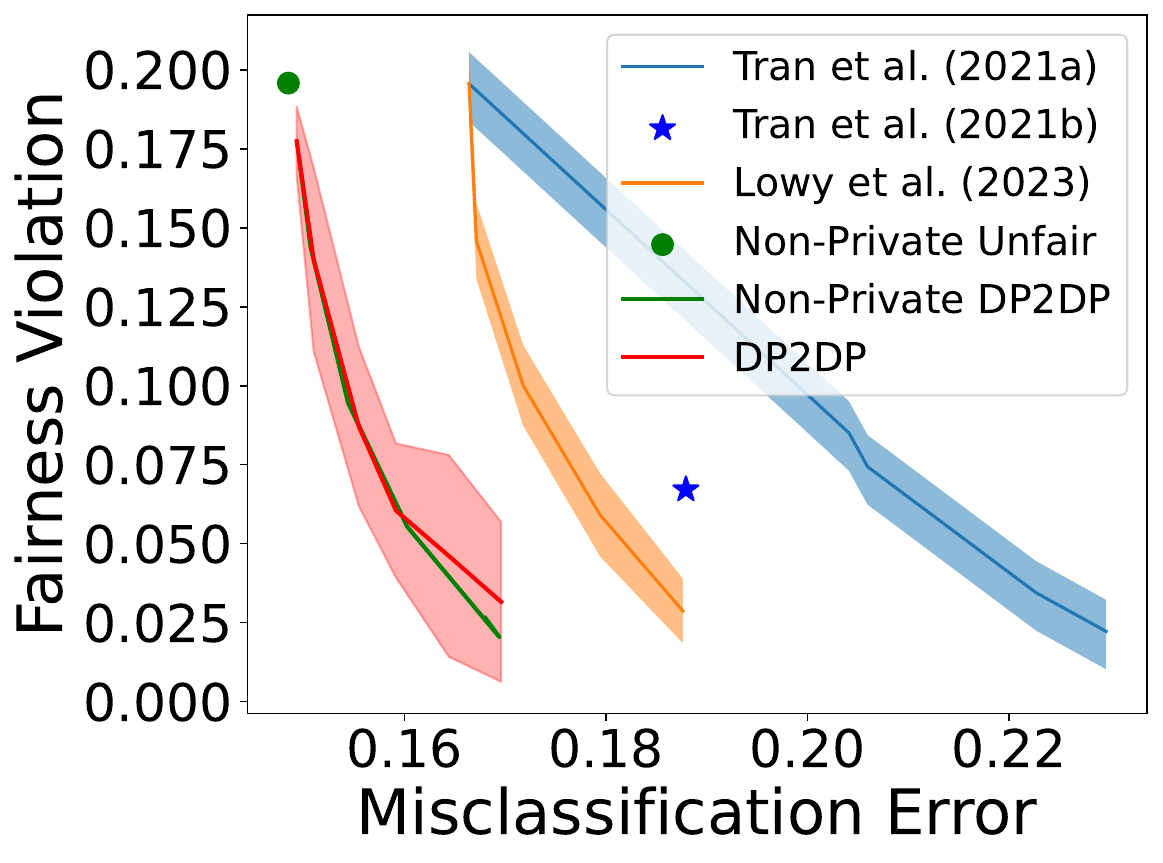}
                    \label{fig:diag_adut_eps05}
                \end{subfigure}
                \begin{subfigure}[b]{0.49\columnwidth}
                    \centering
                    \includegraphics[width=0.8\textwidth]{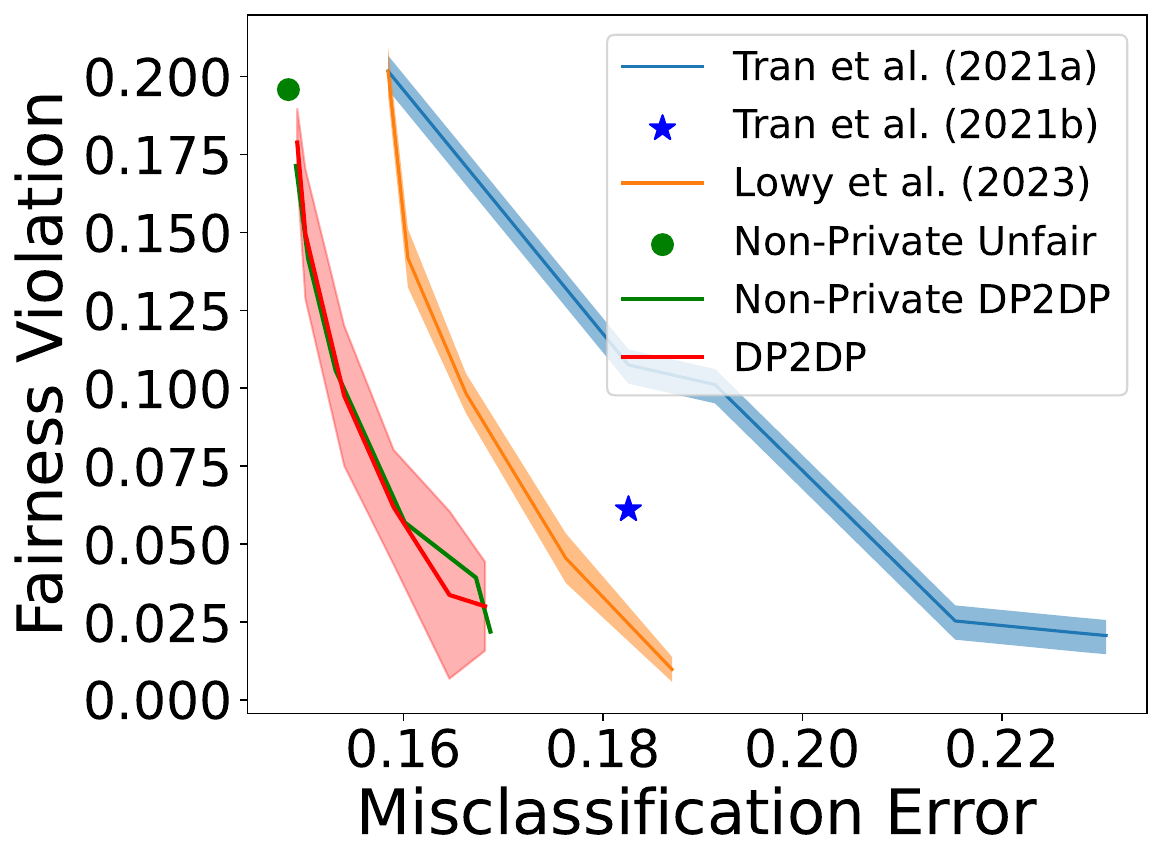}
                    \label{fig:diag_adut_eps1}
                \end{subfigure} 
                \caption{Comparison of our method (in terms of fairness/misclassification) with previous work on the Adult dataset under $(\varepsilon, \delta)$-differential privacy with $\delta = 10^{-5}$. The left panel shows  $\varepsilon = 0.5$, and the right panel for $\varepsilon = 1.0$.}
                \label{fig:diagram_adut}
            \end{figure}

        Figures~\ref{fig:diagram_adut} reports the fairness/accuracy trade-offs on the Adult dataset under two different privacy budgets, $\varepsilon=0.5$ and $\varepsilon=1.0$ for the methods we consider. In both regimes, our DP2DP method consistently outperforms existing baselines achieving strictly better misclassification error at comparable or lower demographic parity violations. This demonstrates that combining fairness post-processing with differential privacy, as proposed in DP2DP, leads to significantly improved practical trade-offs compared to prior approaches. Note that the gap between DP2DP and its non-private counterpart is small across both privacy levels. This observation directly reflects Theorem~\ref{thm:fairness}: as $N$ grows large, the fairness cost of privatization diminishes at rate $\tfrac{\log(N)}{\sqrt{N}}$, ensuring that the private classifier converges to the non-private one. On the Adult dataset, which contains nearly 50,000 examples, this asymptotic effect is already clearly visible in practice. Similar observation hold on the Parkinson \citep{parkinsons_telemonitoring_189} and the Credit Card dataset \citep{default_of_credit_card_clients_350} (see Appendix details on additional experiments).

    \section{Perspectives}
    \label{sec:conclusion}


    In future work, we plan to extend our procedure in several directions. First, handling multiple sensitive attributes is a natural next step, as it would broaden the applicability of our method. Second, we aim to generalize our approach to other fairness notions, such as Equalized Odds, which necessitate the use of labeled data in Phase~2. Finally, addressing the unawareness setting (where only the feature vector $X$ is available at prediction time) remains an important and challenging open direction for future work.


 \newpage  
    

\bibliography{bibfile} 
\bibliographystyle{plainnat}

\clearpage
\appendix
\thispagestyle{empty}


\begin{center}
\Large{\textbf{APPENDIX}}
\end{center}
\bigskip

\section*{ORGANIZATION OF THE APPENDIX}
Appendix~\ref{app:defs} recalls standard analytical definitions.
Appendix~\ref{app:properties} gathers auxiliary technical results useful for the following proofs.
Appendix~\ref{app:privacy} provides the complete privacy analysis of the proposed \textsc{DP2DP} algorithm.
Appendix~\ref{app:fairness} contains the fairness proofs.
Finally, Appendix~\ref{app:exp} presents additional experimental details, including the synthetic data generation procedure and extended empirical results on benchmark datasets.

\section{COMPLEMENTARY DEFINITIONS}\label{app:defs}
    
    This section recalls a few standard definitions and results used in the analysis, in particular regarding Lipschitz continuity and smoothness of differentiable functions.  
    Throughout, we denote by $\lVert \cdot \rVert$ a norm on $\mathbb{R}^n$, and by $\lVert \cdot \rVert_*$ its associated dual norm.
    
    \begin{definition}[$L$-Lipschitz continuity]
        A function $f:\mathbb{R}^n \to \mathbb{R}$ is said to be \emph{$L$-Lipschitz} with respect to $\lVert \cdot \rVert$ if for all $x,y \in \mathbb{R}^n$,
        \begin{equation*}
            |f(x) - f(y)| \le L \lVert x - y \rVert.
        \end{equation*}
        Equivalently, when $f$ is differentiable, this is implied by $\lVert \nabla f(x) \rVert_* \le L$ for all $x \in \mathbb{R}^n$.
    \end{definition}
    
    \begin{definition}[$\beta$-smoothness]
        A differentiable function $f:\mathbb{R}^n \to \mathbb{R}$ is said to be \emph{$\beta$-smooth} with respect to $\lVert \cdot \rVert$ if its gradient is $\beta$-Lipschitz, i.e.,
        \begin{equation*}
            \lVert \nabla f(x) - \nabla f(y) \rVert_* \le \beta \lVert x - y \rVert, \quad \forall x,y \in \mathbb{R}^n.
        \end{equation*}
        Equivalently, when $f$ is twice differentiable, this is implied by $\lVert \nabla^2 f(x) \rVert_{op} \leq \beta$ for all $x \in \mathbb{R}^n$, where the operator norm is induced by $\lVert \cdot \rVert$ and $\lVert \cdot \rVert_*$.
    \end{definition}
    
    \begin{proposition}[Descent lemma {\citep[Theorem~2.1.5]{nesterov_lectures_2018}}]\label{app:prop:descent}
        If $f$ is $\beta$-smooth with respect to $\lVert \cdot \rVert$, then for all $x,y \in \mathbb{R}^n$,
        \begin{equation*}
            f(y) \le f(x) + \langle \nabla f(x), y - x \rangle + \frac{\beta}{2}\lVert y - x \rVert^2.
        \end{equation*}
        Equivalently, in dual form,
        \begin{equation*}
            f(x) + \langle \nabla f(x), y - x \rangle + \frac{1}{2\beta}\lVert \nabla f(x) - \nabla f(y) \rVert_*^2 \le f(y).
        \end{equation*}
    \end{proposition}
    
    \begin{lemma}[Gradient–function gap inequality]\label{app:lemma:ineq-smooth}
        Let $f:\mathbb{R}^n \to \mathbb{R}$ be a differentiable and $\beta$-smooth function. Suppose there exists a global minimizer $x^* \in \mathbb{R}^n$ such that $x^* \in \arg\min_{x \in \mathbb{R}^n} f(x)$. Then, for all $x \in \mathbb{R}^n$,
        \begin{equation*}
            \lVert \nabla f(x) \rVert_*^2 \le 2\beta \bigl(f(x) - f(x^*)\bigr).
        \end{equation*}
    \end{lemma}
    
    \begin{proof}
    Let $x \in \mathbb{R}^n$, since $f$ is $\beta$-smooth, Proposition~\ref{app:prop:descent} gives
    \begin{equation*}
        f(x^*) + \langle \nabla f(x^*), x - x^* \rangle + \frac{1}{2\beta}\lVert \nabla f(x^*) - \nabla f(x) \rVert_*^2 \leq f(x).
    \end{equation*}
    Rearranging terms and use the first order condition $\nabla f(x^*) = 0$ yields
    \begin{equation*}
        \lVert \nabla f(x) \rVert_*^2 \le 2\beta (f(x) - f(x^*)).
    \end{equation*}
    \end{proof}

    \begin{definition}[Operator norms]\label{def:operator-norms}
        Let $\lVert \cdot \rVert_p$ and $\lVert \cdot \rVert_q$ be vector norms on $\mathbb{R}^n$ and $\mathbb{R}^m$, respectively.  
        For $A\in\mathbb{R}^{m\times n}$, the operator norm is
        \begin{equation*}
            \lVert A \rVert_{p, q} := \sup_{v\neq 0}\frac{\lVert Av \rVert_q}{\lVert v \rVert_p} = \inf\{c > 0 : \lVert A v \rVert_q \leq c \lVert v \rVert_p \; \forall v \in \mathbb{R}^n\}.
        \end{equation*}
        We use the following shorthands:
        \begin{equation*}
            \lVert A \rVert_2 := \lVert A \rVert_{2, 2} = \sqrt{\lambda_{\max}(A^\top A)},\qquad
            \lVert A \rVert_1 := \lVert A \rVert_{1, 1}=\max_{j=[n]}\sum_{i=1}^m \lvert a_{ij} \rvert,\qquad
            \lVert A \rVert_\infty := \lVert A \rVert_{\infty, \infty}=\max_{i=[m]}\sum_{j=1}^n \lvert a_{ij} \rvert.
        \end{equation*}
    \end{definition}

\section{TECHNICAL LEMMAS}\label{app:properties}
    In this section, we collect several auxiliary mathematical results that are used throughout the privacy and fairness analyses. These lemmas formalize basic probability result, regularity properties of the per-sample loss, and existence results for the optimization problem introduced in the main paper.
    
    \begin{lemma}[{\citealp[Lemma~4.1]{GyorfiBook_2002}}]\label{app:lemma:binom}
        Let $Z\sim\mathrm{Bin}(N,p)$ and adopt the convention $\tfrac{0}{0}=0$. Then
        \begin{equation*}
            \mathbb{E}\left[\frac{\mathds{1}\left(Z \geq 1\right)}{Z}\right] \leq \frac{2}{(N+1)p}.
        \end{equation*}
    \end{lemma}
    \begin{proof}
        We first compute the auxiliary identity
        \begin{equation}\label{eq:aux}
            \mathbb{E}\left[\frac{1}{Z+1}\right] = \sum_{k=0}^{N}\frac{1}{k+1}\binom{N}{k}p^{k}(1-p)^{N-k} = \frac{1-(1-p)^{N+1}}{(N+1)p}.
        \end{equation}
        
        Indeed, using $\frac{1}{k+1}\binom{N}{k} = \frac{1}{N+1}\binom{N+1}{k+1}$,
        \begin{equation*}
            \sum_{k=0}^{N}\frac{1}{k+1}\binom{N}{k}p^{k}(1-p)^{N-k} = \frac{1}{N+1}\sum_{k=0}^{N}\binom{N+1}{k+1}p^{k}(1-p)^{N-k}.
        \end{equation*}
        
        Write $p^{k}=p^{k+1}/p$ and change index $j=k+1$:
        \begin{equation*}
            \sum_{k=0}^{N}\frac{1}{k+1}\binom{N}{k}p^{k}(1-p)^{N-k} = \frac{1}{(N+1)p}\sum_{j=1}^{N+1}\binom{N+1}{j}p^{j}(1-p)^{(N+1)-j}.
        \end{equation*}
        Now using the binomial formula $(x + y)^n = \sum_{k=1}^{n}\binom{n}{k}x^{k}(y)^{n-k} + y^n, \; \forall x, y \in \mathbb{R}, \; n \in \mathbb{N}$:
        \begin{equation*}
            \sum_{k=0}^{N}\frac{1}{k+1}\binom{N}{k}p^{k}(1-p)^{N-k} = \frac{1}{(N+1)p}\left(1-(1-p)^{N+1}\right),
        \end{equation*}
        
        which proves \eqref{eq:aux}.
        
        Next, for each integer $z\ge 0$ we have the pointwise bound
        \begin{equation}\label{eq:pointwise}
            \frac{\mathds{1}(z\ge 1)}{z} \leq \frac{2\mathds{1}(z\ge 1)}{z+1} \leq \frac{2}{z+1}.
        \end{equation}
        
        Taking expectations in \eqref{eq:pointwise} and using \eqref{eq:aux},
        \begin{equation*}  
            \mathbb{E}\left[\frac{\mathds{1}\left(Z\ge 1\right)}{Z}\right]
            \le 2\,\mathbb{E}\!\left[\frac{1}{Z+1}\right]
            =2\cdot\frac{1-(1-p)^{N+1}}{(N+1)p}
            \le \frac{2}{(N+1)p},
        \end{equation*}
        
        since $1-(1-p)^{N+1}\le 1$.
    \end{proof}

    The next result establishes key regularity properties of the per-sample loss function used in our fairness objective. 
    
    \begin{lemma}\label{app:lemma:sample_loss}
        For all $(x, s) \in \mathcal{X} \times \mathcal{S}$, the per-sample loss $\hat{h}(\cdot; x, s)$ as defined in \eqref{eq:hhat} is convex, $(2\sqrt{2} + \rho\sqrt{2K})$-Lipschitz with respect to $\lVert \cdot \rVert_2$ and $8/\beta$-smooth with respect to $\lVert \cdot \rVert_\infty$.
    \end{lemma}
    \begin{proof}
        Denote $\lambda = (\lambda^{(1)},\lambda^{(2)})^\top \in \mathbb R^{2K}$ and $\mu = (\mu^{(1)},\mu^{(2)})^\top \in \mathbb R^{2K}$.  For all $(x, y) \in \mathcal{X} \times \mathcal{S}$, note we can rewrite
        \begin{equation*}
            \ell^s(x;\lambda) = a^s(x) + J_s \lambda,\qquad
            J_s = \begin{bmatrix}-s I_K & s I_K\end{bmatrix}\in\mathbb R^{K\times 2K}, \qquad a^s(x) = \left(\bar{\pi}_s \bar{p}_1(x, s), \ldots, \bar{\pi}_s \bar{p}_K(x, s)\right)^\top.
        \end{equation*}
        \begin{enumerate}
            \item \emph{Convexity.} Let $(x, y) \in \mathcal{X} \times \mathcal{S}$, $\mathrm{LSE}_\beta$ is convex on $\mathbb{R}^K$ \citep[Examples~3.1.5]{Boyd_Vandenberghe_2004} and $\ell^s(x; \cdot)$ is affine in $\lambda$. Therefore $\mathrm{LSE}_\beta \circ \ell^s(x; \cdot)$ is convex in $\lambda$. The linear regularizer is convex as well, so $\hat{h}(\cdot; x, s)$ is convex.
            
            \item \emph{Lipschitzness.} By the chain rule, for all $(x, y) \in \mathcal{X} \times \mathcal{S}$,
                \begin{equation*}
                    \nabla_\lambda \mathrm{LSE}_\beta(\ell^s(x;\lambda)) = J_s^\top \mathrm{softmax}_\beta(\ell^s(x;\lambda)).
                \end{equation*}
                Hence
                \begin{equation*}
                    \nabla_\lambda \hat{h}(\lambda; x,s) = \lvert \mathcal{S} \rvert J_s^\top \mathrm{softmax}_\beta(\ell^s(x;\lambda)) + \rho \begin{bmatrix}\mathbf{1}_K\\[2pt] \mathbf{1}_K\end{bmatrix}.
                \end{equation*}
                
                 Since $J_s J_s^\top =  2 s^2I_K$, by Definition~\ref{def:operator-norms} we have $\lVert J_s \rVert_2 = \lVert J_s^\top \rVert_2 = \lvert s \rvert\sqrt{2}$. Also, $\lVert \mathrm{softmax}_\beta(\ell^s(x;\lambda)) \rVert_2 \leq 1$, and $s \in \{-1 ,1\}$. Thus for all $(x, y) \in \mathcal{X} \times \mathcal{S}$ we have
                \begin{equation*}
                    \lVert \lvert \mathcal{S} \rvert J_s^\top \mathrm{softmax}_\beta(\ell^s(x;\lambda)) \rVert_2 \leq \lvert \mathcal{S} \rvert \, \lVert J_s^\top\rVert_{2} \lVert \mathrm{softmax}_\beta(\ell^s(x;\lambda)) \rVert_2 \leq \lvert \mathcal{S} \rvert \lvert s \rvert\sqrt{2} = 2 \sqrt{2}.
                \end{equation*}
                The regularizer contributes $\rho \left\lVert \begin{smallmatrix} \mathbf{1}_K \\ \mathbf{1}_K \end{smallmatrix} \right\rVert_2 = \rho\sqrt{2K}$. Combining gives the stated Lipschitz constant.

                \begin{align*}
                    \left\lVert \nabla_\lambda \hat{h}(\lambda; x,s)\right\rVert_2 &= \left\lVert \lvert \mathcal{S} \rvert J_s^\top \mathrm{softmax}_\beta(\ell^s(x;\lambda)) + \rho \begin{bmatrix}\mathbf{1}_K\\[2pt] \mathbf{1}_K\end{bmatrix} \right\rVert_2 \\
                    &\leq \left\lVert \lvert \mathcal{S} \rvert J_s^\top \mathrm{softmax}_\beta(\ell^s(x;\lambda))\right\rVert_2 + \rho \left\lVert\begin{matrix}\mathbf{1}_K\\[2pt] \mathbf{1}_K\end{matrix} \right\rVert_2 \\
                    &\leq 2 \sqrt{2} + \rho\sqrt{2K}.
                \end{align*}
                
            \item \emph{Smoothness.}  Throughout, recall that $\mathrm{LSE}_\beta$ is $\tfrac{1}{\beta}$-smooth, i.e., its gradient, $\mathrm{softmax}_\beta$, is $\tfrac{1}{\beta}$-Lipschitz from $(\mathbb{R}^{K}, \lVert \cdot \rVert_\infty)$ to $(\mathbb{R}^{K}, \lVert \cdot \rVert_1)$. Moreover, by Definition~\ref{def:operator-norms}, $\lVert J_s \lVert_\infty = 2 \lvert s \rvert$ and $\lVert J_s^\top \lVert_1 = \lVert J_s \lVert_\infty = 2\lvert s \rvert$. 
            Therefore, we have for all $(x, y) \in \mathcal{X} \times \mathcal{S}$,
            \begin{equation*}
                 \nabla_\lambda \hat{h}(\lambda;x,s)-\nabla_\lambda \hat{h}(\mu;x,s) = \lvert \mathcal{S} \rvert J_s^\top \left(\mathrm{softmax}_\beta(\ell^s(x;\lambda))-\mathrm{softmax}_\beta(\ell^s(x;\mu))\right).
            \end{equation*}
            Taking $\ell_1$ norms and using $\lVert A v \rVert_1 \leq \lVert A \rVert_1 \lVert v \rVert_1 \; \forall v$,
            \begin{equation*}
                \left\lVert\nabla_\lambda \hat{h}(\lambda;x,s)-\nabla_\lambda \hat{h}(\mu;x,s)\right\rVert_1
                \leq \lvert\mathcal{S}\rvert \lVert J_s^\top \rVert_{1} \left\lVert\mathrm{softmax}_\beta(\ell^s(x;\lambda))-\mathrm{softmax}_\beta(\ell^s(x;\mu))\right\rVert_{1}.
            \end{equation*}
            By $\tfrac{1}{\beta}$-Lipschitzness of $\mathrm{softmax}_\beta$,
            \begin{equation*}
                \left\lVert \mathrm{softmax}_\beta(\ell^s(x;\lambda))-\mathrm{softmax}_\beta(\ell^s(x;\mu)) \right\rVert_{1} \leq \tfrac{1}{\beta} \left\lVert \ell^s(x;\lambda)-\ell^s(x;\mu) \right\rVert_{\infty}.
            \end{equation*}
            Since $\ell^s(x;\cdot)$ is linear with  $J_s$, we have
            \begin{equation*}
                \left\lVert \ell^s(x;\lambda)-\ell^s(x;\mu) \right\rVert_{\infty}
                = \lVert J_s(\lambda-\mu) \rVert_{\infty} \leq \lVert J_s \rVert_{\infty} \lVert \lambda-\mu \rVert_{\infty}.
            \end{equation*}
            Combining the above,
            \begin{equation*}
                \bigl\|\nabla_\lambda \hat{h}(\lambda;x,s)-\nabla_\lambda \hat{h}(\mu;x,s)\bigr\|_1
                \le \frac{|\mathcal{S}|}{\beta}\,\|J_s^\top\|_{1}\,\|J_s\|_{\infty}\,\|\lambda-\mu\|_{\infty}.
            \end{equation*}
            Using $\lVert J_s^\top \rVert_{1} = \lVert J_s \rVert_{\infty}=2 \lvert s \rvert$ and $\lvert \mathcal{S} \rvert = 2$ with $\lvert s \rvert = 1$,
            \begin{equation*}
                \left\lVert \nabla_\lambda \hat{h}(\lambda;x,s)-\nabla_\lambda \hat{h}(\mu;x,s) \right\rVert_1
                \leq \frac{8}{\beta} \lVert \lambda-\mu \rVert_{\infty}.
            \end{equation*}
            Thus, $\nabla_\lambda \hat{h}(\cdot;x,s)$ is $\tfrac{8}{\beta}$-Lipschitz.
        \end{enumerate} 
    \end{proof}
    
    The following lemma establishes the existence of a solution to the empirical objective $\hat{H}_\beta$.
    \begin{lemma}[Existence of a minimizer]\label{app:lemma:minimum}
        There exists at least one solution, that lies in a compact subset $[0,C_\lambda]^{2K}\subset\mathbb{R}^{2K}_+$ for some $C_\lambda>0$, to the problem $\min_{\lambda\in\mathbb{R}^{2K}_+} \hat{H}_\beta(\lambda)$, where
        \begin{equation*}
            \hat{H}_\beta(\lambda^{(1)},\lambda^{(2)})
            =\sum_{s\in\mathcal S}\frac{1}{N_s}\sum_{x\in D_{N\mid\mathcal X,s}}
            \mathrm{LSE}_\beta \left(\ell^s(x;\lambda)\right)
            + \rho\sum_{k=1}^K\left(\lambda^{(1)}_k+\lambda^{(2)}_k\right).
        \end{equation*}
    \end{lemma}
    \begin{proof}
        First, $\hat{H}_\beta$ is continuous and convex in $(\lambda^{(1)},\lambda^{(2)})$ from Lemma~\ref{app:lemma:sample_loss}.  
        We now show that $\hat{H}_\beta$ is coercive.
        
        Recall that for any $x\in\mathbb{R}^n$,
        \begin{equation*}
            \max_{i\in[n]} x_i \leq \mathrm{LSE}_\beta(x)
            \leq \max_{i\in[n]} x_i + \log n .
        \end{equation*}
        Therefore,
        \begin{equation*}
            \sum_{s \in \mathcal{S}} \frac{1}{N_s}\sum_{x\in D_{N \mid \mathcal{X},s}}
            \mathrm{LSE}_\beta\left(\ell^s(x;\lambda)\right) \geq \sum_{s \in \mathcal{S}} \frac{1}{N_s}\sum_{x\in D_{N\mid\mathcal X,s}}
            \max_{k\in[K]}\left\{\bar\pi_s \bar p_k(x,s) - s \left(\lambda^{(1)}_k - \lambda^{(2)}_k\right)\right\}.
        \end{equation*}
        Using that averages of maxima dominate maxima of averages and then exchanging the order of
        $\max$ and $\sum$,
        \begin{align*}
            \sum_{s \in \mathcal{S}} \frac{1}{N_s} \sum_{x\in D_{N \mid \mathcal{X},s}} \max_{k}\{\bar{\pi}_s \bar{p}_k(x,s) - s(\lambda^{(1)}_k - \lambda^{(2)}_k)\} &\geq \sum_{s \in \mathcal{S}} \max_{k}\left\{ \frac{1}{N_s} \sum_{x\in D_{N \mid \mathcal{X},s}} \left(\bar{\pi}_s \bar{p}_k(x,s) - s(\lambda^{(1)}_k - \lambda^{(2)}_k) \right)\right\}\\
            &\geq \max_{k}\left\{ \sum_{s \in \mathcal{S}} \left(\frac{ \bar{\pi}_s}{N_s} \sum_{x\in D_{N \mid \mathcal{X},s}} \bar{p}_k(x,s) - s (\lambda^{(1)}_k - \lambda^{(2)}_k)\right)\right\}\\
        \end{align*}
        Since $s \in \{-1, 1\}$ we have
        \begin{align*}
            \sum_{s \in \mathcal{S}} \frac{1}{N_s} \sum_{x\in D_{N \mid \mathcal{X},s}} \max_{k}\{\bar{\pi}_s \bar{p}_k(x,s) - s(\lambda^{(1)}_k - \lambda^{(2)}_k)\} &\geq \max_{k}\left\{\sum_{s \in \mathcal{S}} \frac{\bar{\pi}_s}{N_s} \sum_{x\in D_{N \mid \mathcal{X},s}} \bar{p}_k(x,s)\right\}\\
            &\geq 0.
        \end{align*}
        Hence,
        \begin{equation*}
            \hat H_\beta(\lambda^{(1)},\lambda^{(2)}) \geq \rho\sum_{k=1}^K\left(\lambda^{(1)}_k + \lambda^{(2)}_k\right),
        \end{equation*}
        which shows that $\hat H_\beta$ is coercive on $\mathbb{R}^{2K}_+$. By standard results for
        lower semicontinuous, coercive, convex functions, a minimizer $(\hat\lambda^{(1)},\hat\lambda^{(2)})$ exists and lies in a compact subset $[0,C_\lambda]^{2K} \subset \mathbb{R}^{2K}_+$ for some $C_\lambda>0$.
    \end{proof}

    \subsection{DP-SGD}
        We recall below the standard DP-SGD algorithm and two key theoretical guarantees used in our analysis.
        \begin{algorithm}[H]
            \caption{Differentially Private Stochastic Gradient Descent (DP-SGD)}
            \label{app:algo:dp-sgd}
            \KwInput{
                Examples $\{(x_1, y_1), \ldots, (x_n, y_n)\} = \{z_1, \ldots, z_n\}$, loss function $\mathcal{L}(\theta) =  \tfrac{1}{n} \sum_{i=1}^n \ell(\theta; z_i)$. Parameters: learning rate $\eta_t$, noise scale $\sigma$, group size $b$.
            }
            \textbf{Initialize} $\theta_0$ randomly \\
            
            \For{$t \in [T]$}{
                Take a random sample $b_t$ with sampling probability $b/n$ \\
                \textbf{Compute gradient}\\
                For each $i \in b_t$, compute $g_t(z_i) \gets \nabla_{\theta^t}\ell(\theta_t; z_i)$ \\
                \textbf{Add noise}\\
                $\tilde{g}_t \gets \frac{1}{b}\left(\sum_{i \in b_t} g_t(z_i) + \mathcal{N}(0, \sigma^2 I_p)\right)$\\
                \textbf{Descent}\\
                $\theta_{t+1} \gets \Pi_{\Theta}(\theta_t - \eta_t \tilde{g}_t)$
            }
            \KwOutput{
                $\theta_T$
            }
        \end{algorithm} 
            
        \begin{theorem}[RDP for DP-SGD \cite{altschuler2023}]\label{app:theorem:soft-dpsgd}
            Let $\Theta \subset \mathbb{R}^p$ be a convex set of diameter bounded by $D$, and consider optimizing convex losses $\ell$ over $\Theta$ that are $M$-smooth and have a $\ell_2$-sensitivity $\Delta_2 < \infty$. For any number of iterations $T$, dataset size $n \in \mathbb{N}$, batch size $b \leq n$, stepsize $\eta \leq 2/M$, noise parameter $\sigma > 0$, and initialization $\theta_0 \in \Theta$, DP-SGD satisfies $(\alpha, \varepsilon)$-RDP for all $\alpha \geq 1$ and
            \begin{equation*}
                \varepsilon \leq \min\left\{T Q , \min_{\underset{\sigma_1^2 + \sigma_2^2 = \sigma^2}{\sigma_1, \sigma_2 > 0} } \min_{\tilde{T} \in [T-1]} \tilde{T}Q + \frac{\alpha D^2}{2 \eta^2 \sigma_1^2 \tilde{T}} \right\}
            \end{equation*}
            with $Q = S_\alpha\!\left(\frac{b}{n}, \frac{b \sigma_2}{\Delta_2}\right)$ where $S_\alpha\!\left(q, \sigma \right) := D_\alpha(\mathcal{N}(0, \sigma^2) \Vert (1-q)\mathcal{N}(0, \sigma^2) + q \mathcal{N}(1, \sigma^2))$.
        \end{theorem}

        \begin{theorem}[Utility guarantees for DP-SGD \cite{bassily2014privateerm}]\label{app:lemma:utility_dpsgd}
            Let $\Theta \subset \mathbb{R}^p$ be a convex domain of diameter bounded by $D$, and let the loss function $\ell$ be convex and $L$-Lipschitz over $\Theta$. For $T > 0$ and $\gamma_t = \frac{D}{\sqrt{t\left(L^2 + \tfrac{\sigma^2}{b^2}p\right)}}$, there exists a constant C such that DP-SGD guarantees:
            \begin{equation*}
                \mathbb{E}[\mathcal{L}(\theta_T)] - \min_{\theta \in \Theta}\mathcal{L}(\theta) \leq C\left(\tfrac{D \log T}{\sqrt{T}}(L + \tfrac{\sigma}{b}p)\right)
            \end{equation*}
        \end{theorem}

\section{DETAILS ON THE PRIVACY ANALYSIS}\label{app:privacy}
    
    In this section, we provide the detailed privacy guarantees supporting Theorem~\ref{theorem:rdp-algo}. We begin by quantifying the sensitivity of the quantities released, then combine them via Rényi differential privacy composition.
    
    \subsection{Proof Outline}
        The overall privacy proof proceeds in three stages:
        \begin{enumerate}
            \item Bounding the sensitivity and privacy cost of the privatized group proportions;
            \item Deriving the Rényi differential privacy guarantee for the DP-SGD updates in Phase~2;
            \item Combining both mechanisms through standard Rényi differential privacy composition.
        \end{enumerate}
        
        \subsubsection{Group Proportions}\label{app:subsubsec:pi}
            We first analyze the sensitivity and privacy guarantees of the group proportion estimates $(\bar{\pi}_s)_{s \in \mathcal{S}}$.
            
            \begin{lemma}\label{app:lemma:sensi-pi}
                The $\ell_2$-sensitivity of $\phi_s : D_N \to \hat{\pi}_s$ is $1/N$ for each $s \in \mathcal{S}$.
            \end{lemma}
            \begin{proof}
                For all $s \in \mathcal{S}$, recall that $\hat{\pi}_s = \tfrac{1}{N}\sum_{i=1}^N \mathds{1}(\tilde{s}_i = s)$. 
                Let $D_N = \{(\tilde{x}_j, \tilde{s}_j) \; \vert \; j \in [N]\}$ and $\mathcal{D}'_N = \{(\tilde{x}'_j, \tilde{s}'_j) \; \vert \; j \in [N]\}$ be two neighboring datasets that differ in exactly one entry, say at index $j$.
                All terms cancel in the difference of the two estimates except the $j$-th term, i.e.,
                \begin{align*}
                    \Delta_2(\phi_s) &= \left\lvert \frac{1}{N}\sum_{i=1}^N \mathds{1}(\tilde{s}_i=s) - \frac{1}{N}\sum_{i=1}^N \mathds{1}(\tilde{s}'_i=s) \right\rvert \\
                    &= \frac{1}{N}\, \bigl\lvert \mathds{1}(\tilde{s}_j=s) - \mathds{1}(\tilde{s}'_j=s) \bigr\rvert.
                \end{align*}
                
                Since the indicator takes values in $\{0,1\}$, the absolute difference is at most $1$, and this bound is attained when the changed record switches membership. Hence $\Delta_2(\phi_s) \leq 1/N$, with equality in the worst case.
            \end{proof}
            \begin{lemma}\label{app:lemma:rdp-pi}
                Releasing $(\bar{\pi}_s)_{s \in \mathcal{S}}$ satisfies $(\alpha, \frac{\alpha}{2 N^2 \sigma_\pi^2})$-Rényi differential privacy for all $\alpha \geq 1$.
            \end{lemma}
            \begin{proof}
                For each $s\in\mathcal S$, Lemma~\ref{app:lemma:sensi-pi} and \cite[Corollary~3]{rdp} imply that $\bar\pi_s$ satisfies $(\alpha,\tfrac{\alpha}{2N^2\sigma_\pi^2})$-Rényi differential privacy.  
                Let $\varepsilon_s$ denote the RDP parameter of $\bar\pi_s$ for each $s \in \mathcal{S}$. Since the $\{\bar\pi_s\}_{s\in\mathcal S}$ are computed on disjoint subsets of $D_N$, the \emph{parallel composition} property applies: releasing the tuple $(\bar\pi_s)_{s\in\mathcal S}$ is $(\alpha,\max_{s\in\mathcal S}\varepsilon_s)$-Rényi differential private. In our case $\varepsilon_s = \tfrac{\alpha}{2N^2\sigma_\pi^2}$ for all $s$, hence the joint release is $(\alpha,\tfrac{\alpha}{2N^2\sigma_\pi^2})$-Rényi differential private.
            \end{proof}

        \subsubsection{DP-SGD}
            
            We now turn to the DP-SGD procedure used to optimize the smoothed fairness objective. We first bound the sensitivity of the per-sample gradient.
            \begin{lemma}\label{app:lemma:sensi-gradient}
                Let $\lambda \in [0, C_\lambda]^{K}$. The $\ell_2$-sensitivity of the per-sample gradient $\nabla \hat{h}(\lambda; \cdot)$ is at most $4$.
            \end{lemma}
            \begin{proof}
                Fix two distinct samples $(x,s) \in \mathcal{X} 
                \times \mathcal{S}$ and $(x',s') \in \mathcal{X} 
                \times \mathcal{S}$ and consider the gradient with respect to $\lambda^{(l)} \in [0, C_\lambda]^{K}$ with $l \in \{1, 2\}$.
                By definition, the $\ell_2$-sensitivity is
                \begin{equation*}
                    \Delta_2(\nabla_{\lambda^{(l)}} \hat{h}(\lambda; \cdot)) = \left\lVert \nabla_{\lambda^{(l)}} \hat{h}(\lambda;x,s) - \nabla_{\lambda^{(l)}} \hat{h}(\lambda;x',s') \right\rVert_2.
                \end{equation*}
                
                Since for any $(x, s) \in \mathcal{X} \times \mathcal{S}$,
                \begin{equation*}
                    \nabla_{\lambda^{(l)}} \hat{h}(\lambda; x, s) = (2 l - 3) \lvert \mathcal{S} \rvert \; s \;\mathrm{softmax}_\beta(\ell^s(x;\lambda)) + \rho \mathbf{1}_K.
                \end{equation*}
                
                The $\ell_2$-sensitivity reduces to a scaled difference of two softmax probability vectors:
                \begin{align*}
                    \Delta_2(\nabla_{\lambda^{(l)}} \hat{h}(\lambda; \cdot)) &= \left\lVert (2 l - 3)\lvert\mathcal{S}\rvert \; s \; \mathrm{softmax}_\beta(\ell^s(x,\lambda)) - (2 l - 3) \lvert \mathcal{S} \rvert \; s' \;\mathrm{softmax}_\beta(\ell^{s'}(x',\lambda)) \right\rVert_2 \\
                    &= \lvert\mathcal{S}\rvert \left\lVert \; s \; \mathrm{softmax}_\beta(\ell^s(x,\lambda)) - \; s' \; \mathrm{softmax}_\beta(\ell^{s'}(x',\lambda)) \right\rVert_2.
                \end{align*}
                Applying the triangle inequality and homogeneity of the norm, with $s \in \{-1, 1\}$, yields
                \begin{equation*}
                    \Delta_2(\nabla_{\lambda^{(l)}} \hat{h}(\lambda; \cdot)) \leq \lvert\mathcal{S}\rvert \left( \lVert  \mathrm{softmax}_\beta(\ell^s(x,\lambda) \rVert_2 + \lVert \mathrm{softmax}_\beta(\ell^{s'}(x',\lambda)) \rVert_2 \right).
                \end{equation*}
                Since $\mathrm{softmax}_\beta(\cdot)$ lies in the simplex, $\lVert \mathrm{softmax}_\beta(x) \rVert_1 \leq 1$ for all $x \in \mathbb{R}$. Thus, using $\lVert x \rVert_2 \leq \lVert x \rVert_1$ for all $x \in \mathbb{R}$, each term has $\ell_2$-norm at most $1$. Therefore,
                \begin{equation*}
                    \Delta_2(\nabla_{\lambda^{(l)}} \hat{h}(\lambda; \cdot)) \leq 2\,|\mathcal{S}|.
                \end{equation*}
                With $|\mathcal{S}|=2$, we conclude $\Delta_2(\nabla_{\lambda^{(l)}} \hat{h}(\lambda; \cdot)) \leq 4$.
            \end{proof}
            
            
        \subsubsection{Composition}
            \begin{lemma}[Rényi differential privacy composition, \cite{rdp}]\label{app:lemma:compo}
                Let $\alpha \geq 1$ and $\varepsilon_1, \varepsilon_2 \geq 0$. Let $\mathcal{M}_1: \mathcal{Z}^m \to \mathcal{V}_1$ and $\mathcal{M}_2: \mathcal{V}_1 \times \mathcal{Z}^m \to \mathcal{V}_2$ be two randomized mechanisms such that:
                \begin{equation*}
                    \mathcal{M}_1 \text{ is } (\alpha, \varepsilon_1)\text{-Rényi differential private}, \quad \mathcal{M}_2 \text{ is } (\alpha, \varepsilon_2)\text{-Rényi differential private}.
                \end{equation*}
                Then the composition $M_2 \circ M_1$ is $(\alpha, \varepsilon_1 + \varepsilon_2)$-Rényi differential private.
            \end{lemma}

    \subsection{Proof of Theorem~\ref{theorem:rdp-algo}}
        We recall the theorem statement below for convenience.
        \begin{theorem}\label{app:theorem:rdp-algo} 
            Consider the \textsc{DP2DP} scheme, as in Algorithm~\ref{algo:dp-fair}. If the step-size sequence is such that $\eta_t = \eta \leq \beta/4$ for all $t \in [T]$, then \textsc{DP2DP} satisfies $(\alpha, \varepsilon)$-Rényi differential privacy for all $\alpha \geq 1$, where
            \begin{equation*}
                \varepsilon \leq \frac{\alpha}{2 N^2 \sigma_\pi^2} + \Psi\left(T, b ,N, \eta, \sigma_{\textsc{SGD}} \right)
            \end{equation*}
            where $\Psi := \Psi\left(T, b ,N, \eta, \sigma_{\textsc{SGD}} \right) $ is defined as
            \[ \Psi = { \min\left\{T Q , \min_{\underset{\sigma_1^2 + \sigma_2^2 = \sigma_{\textsc{sgd}}^2}{\sigma_1, \sigma_2 > 0} } \min_{M \in [T-1]} M Q + \frac{\alpha 2K C_\lambda^2}{2 \eta^2 \sigma_1^2 M } \right\} }, \]
            $Q = S_\alpha\!\left(\frac{b}{N}, \frac{b \sigma_2}{4}\right)$, and for any $(q, \sigma) \in [0,1]\times \mathbb{R}_+$ we define $S_\alpha\!\left(q, \sigma \right)$ as the Rényi divergence of level $\alpha$ between a Gaussian distribution $\mathcal{N}(0, \sigma^2)$ and a mixture of Gaussian $(1-q)\mathcal{N}(0, \sigma^2) + q \mathcal{N}(1, \sigma^2)$.
        \end{theorem}
        \begin{proof}
            The release of the privatized group proportions $(\hat{\pi}_s^{\mathrm{priv}})_{s \in \mathcal{S}}$ satisfies $(\alpha, \tfrac{\alpha}{2 N^2 \sigma_\pi^2})$-Rényi differential privacy by Lemma~\ref{app:lemma:rdp-pi}.
            
            For the DP-SGD component, we directly apply the general result of Theorem~\ref{app:theorem:soft-dpsgd}\citep{altschuler2023}, which provides an $(\alpha, \varepsilon)$-Rényi differential privacy bound for convex losses with $\ell_2$-sensitivity $\Delta_2$. In our setting, Lemma~\ref{app:lemma:minimum} ensures that $\Theta = [0, C_\lambda]^{2K}$, so the domain diameter is $D = \sqrt{2K} C_\lambda$, Lemma~\ref{app:lemma:sample_loss} gives $M = 8/\beta$, and Lemma~\ref{app:lemma:sensi-gradient} obtains $\Delta_2 = 4$. Substituting these quantities into Theorem~\ref{app:theorem:soft-dpsgd} yields the term $\Psi(T, b, N, \eta, \sigma_{\textsc{SGD}})$ defined above.
            
            Finally, by the composition property of Rényi differential privacy (Lemma~\ref{app:lemma:compo}), combining the release of $(\hat{\pi}_s^{\mathrm{priv}})_{s \in \mathcal{S}}$ with the DP-SGD updates results in the overall guarantee
            \begin{equation*}
                (\alpha,\, \tfrac{\alpha}{2 N^2 \sigma_\pi^2} + \Psi(T, b, N, \eta, \sigma_{\textsc{SGD}}))\text{-Rényi differential privacy}.
            \end{equation*}
            This completes the proof.
        \end{proof}

    \subsection{Details on the Discussion}
        From Lemma~\ref{app:lemma:rdp-pi}, releasing $(\hat{\pi}_s^{\mathrm{priv}})_{s \in \mathcal{S}}$ satisfies $(\alpha, \frac{\alpha}{2 N^2 \sigma_\pi^2})$-Rényi differential privacy guarantee. Next, applying \cite[Theorem~3.1]{altschuler2023}, we find that our DP-SGD procedure satisfies $\left(\alpha, \frac{16\alpha}{N^2 \sigma_{\textsc{sgd}}^2} \min\left\{T, \left\lceil \frac{C_\lambda \sqrt{2K} N}{4 \beta}\right\rceil\right\}\right)$-Rényi differential privacy. By the composition Lemma~\ref{app:lemma:compo}, Algorithm~\ref{algo:dp-fair} therefore satisfies $(\alpha, \epsilon_\alpha)$-Rényi differential privacy with 
        \begin{equation*}
            \epsilon_\alpha= \frac{\alpha}{N^2}\left(\frac{1}{2 \sigma_\pi^2} + \frac{16}{\sigma_{\textsc{sgd}}^2} \min\left\{T, \left\lceil \frac{C_\lambda \sqrt{2K} N}{4 \beta}\right\rceil\right\}\right).
        \end{equation*} 
        
        Applying \cite[Proposition~3]{rdp}, we convert this to an $(\epsilon_\delta, \delta)$-DP guarantee for any $\delta > 0$, where $\epsilon_\delta = \min_{\alpha > 1}\left(\epsilon_\alpha + \frac{\log(1/\delta)}{\alpha - 1}\right)$. Solving this optimization yields $\epsilon_\delta = K + 2\sqrt{K \log(1/\delta)}$ where 
        \begin{equation*}
            K = \frac{1}{N^2}\left(\frac{1}{2 \sigma_\pi^2} + \frac{16}{\sigma_{\textsc{sgd}}^2} \min\left\{T, \left\lceil \frac{C_\lambda \sqrt{2K} N}{4 \beta}\right\rceil\right\}\right). 
        \end{equation*}
        
        Finally, by substituting 
        \begin{equation*}
            \sigma_{\textsc{sgd}}^2 = \frac{16 \min\left\{T, \left\lceil \frac{C_\lambda \sqrt{2K} N}{4 \beta}\right\rceil\right\}}{N^2 \left(\sqrt{log(1/\delta) + \epsilon} - \sqrt{log(1/\delta)}\right)^2 - 1/2 \sigma_\pi^2}
        \end{equation*} 
        we obtain $\epsilon_\delta = \epsilon$.
        
\section{FAIRNESS PROOFS}\label{app:fairness}
    In this section, we provide the detailed proofs supporting the fairness guarantees. We begin by introducing a series of intermediate lemmas that capture the smoothing, probabilistic, and optimization properties used in the analysis. Throughout, we rely on several auxiliary results established in Appendix~\ref{app:properties}, which describe regularity and boundedness properties of the function $\hat{H}_\beta$ underlying our fairness constraints. For completeness, we recall several notational conventions used in this section. 
        
    Throughout, we will work with expectations over the sampling process, and therefore treat the dataset as random. To simplify notation, we denote by $D_N = \{(X_i, S_i)\}_{i=1}^N$ the collection of $N$ i.i.d.\ random variables drawn from the joint distribution $\mathbb{P}_{(X,S)}$ over $\mathcal{X} \times \mathcal{S}$. With a slight abuse of notation, we use the same symbol $D_N$ to denote both the random sample and a particular realization of it.
    
    For each sensitive attribute value $s \in \mathcal{S}$, let $\mathbb{P}_{X \mid S=s}$ denote the conditional distribution of $X$ given $S=s$, and define the corresponding subset of the sample  $D_{N \mid \mathcal{X}, s} = \{ X_i \in \mathcal{X} \mid (X_i, S_i) \in D_N, \, S_i = s \}.$ Its empirical counterpart is written, for any measurable set $A \subseteq \mathcal{X}$, 
    \begin{equation*}
        \hat{\mathbb{P}}_{X \mid S=s}(A) = \frac{1}{N_s} \sum_{X_i \in D_{N \mid \mathcal{X}, s}} \mathds{1}\left(X_i \in A\right),
    \end{equation*} and the corresponding empirical expectation as $\hat{\mathbb{E}}_{X \mid S=s}[f(X)] = \frac{1}{N_s} \sum_{X_i \in D_{N \mid \mathcal{X}, s}} f(X_i)$, where $N_s = \sum_{i=1}^N \mathds{1}\left(S_i = s\right)$ denotes the number of samples in group $s$. Under i.i.d. sampling with group prior $\pi_s = \mathbb{P}(S = s)$, each $N_s$ follows a binomial distribution $\mathrm{Bin}(N, \pi_s)$, and we define the smallest group size as $N_{\min} := \min\{N_{-1}, N_{1}\}$.
    
    The deviation between the population and empirical conditional measures is written, for any measurable set $A \subseteq \mathcal{X}$,
    \begin{equation*}
    \bigl(\mathbb{P}_{X \mid S=s} - \hat{\mathbb{P}}_{X \mid S=s}\bigr)(A)
    := \mathbb{P}_{X \mid S=s}(A) - \hat{\mathbb{P}}_{X \mid S=s}(A).
    \end{equation*}
    
    Finally, for each $x \in \mathcal{X}$, $s \in \mathcal{S}$, and $k \in [K]$, we recall that $\ell^s_k(x; \lambda) = \bar{\pi}_s \bar{p}_k(x,s) - s\bigl(\lambda^{(1)}_k - \lambda^{(2)}_k\bigr)$, and $\ell^s(x; \lambda)= \bigl(\ell^s_k(x; \lambda)\bigr)_{k \in [K]}.$

    \subsection{Proof Outline}
        The proof of Theorem~\ref{thm:fairness} follows three main steps:
        \begin{enumerate}
            \item Controlling the smooth approximation error introduced by the $\mathrm{LSE}_\beta$ relaxation;
            \item Bounding the deviation of empirical quantities from their population counterparts;
            \item Accounting for the optimization and privacy noise due to the DP-SGD procedure.
        \end{enumerate}
        We detail each step below.
        
        \subsubsection{Smooth Approximation}
            First, we control the approximation error introduced by replacing the maximum operator with its $\mathrm{LSE}_\beta$ relaxation. This lemma implies that as $\beta \to 0$, the softmax converges exponentially fast to the argmax in sup-norm, with the rate controlled by the margin $\gamma$.

            \begin{lemma}\label{app:lemma:softmax}
                Let $x \in \mathbb{R}^n$, define $\mathcal{A} = \arg\max_{j \in [n]} x_j, \; M = \max_{j \in [n]} x_j$, and for $\beta > 0$ set
                \begin{equation*}
                    \gamma(x) := 
                    \begin{cases}
                    \min_{j \notin \mathcal{A}} \bigl(M - x_j\bigr), & \text{if } \mathcal{A} \neq [n], \\
                    +\infty, & \text{if } \mathcal{A} = [n].
                    \end{cases}
                \end{equation*}
                Then for every $i \in [n]$,
                \begin{equation*}
                    \left\lvert \mathrm{softmax}_\beta(x)_i - \frac{1}{|\mathcal{A}|}\mathds{1}(i \in \mathcal{A}) \right\rvert \leq \max\left\{\frac{1}{\lvert \mathcal{A} \rvert}, \frac{n - \lvert \mathcal{A} \rvert}{\lvert \mathcal{A} \rvert^2}\right\} e^{-\gamma(x) / \beta}.
                \end{equation*}
            \end{lemma}
            \begin{proof}
                Let $x \in \mathbb{R}^n$ and $\beta >0$. By definition of the softmax function, for every $i \in [n]$, we have
                \begin{align*}
                    \mathrm{softmax}_\beta(x)_i 
                    &= \frac{\exp(x_i/\beta)}{\sum_{j=1}^n \exp(x_j/\beta)} \\
                    &= \frac{\exp(x_i/\beta)}{\sum_{j \in \mathcal{A}} \exp(x_j/\beta) + \sum_{j \notin \mathcal{A}} \exp(x_j/\beta)}.
                \end{align*}
                Since each $j \in \mathcal{A}$ satisfies $x_j = M$, this becomes
                \begin{equation*}
                    \mathrm{softmax}_\beta(x)_i = \frac{\exp(x_i/\beta)}{\lvert \mathcal{A} \rvert\exp(M/\beta) + \sum_{j \notin \mathcal{A}} \exp(x_j/\beta)} 
                \end{equation*}

                By multiplying by $\exp(-M/\beta)$ on the denominator and the numerator we get
                \begin{align*}
                    \mathrm{softmax}_\beta(x)_i &= \frac{\exp((x_i - M)/\beta)}{\lvert \mathcal{A} \rvert + \sum_{j \notin \mathcal{A}} \exp((x_j - M)/\beta)} \\
                    &= \frac{\exp\left(-\tfrac{M - x_i}{\beta}\right)}{\lvert \mathcal{A} \rvert + \sum_{j \notin \mathcal{A}} \exp\left(-\tfrac{M - x_j}{\beta}\right)}.
                \end{align*}
                
                Let us now distinguish two cases: (i) the case when $i \in \mathcal{A}$ and (ii) the case when $i \notin \mathcal{A}$.
                \begin{enumerate}[label=(\roman*)]
                    \item If $i \in \mathcal{A}$, then $M - x_i = 0$, hence 
                    \begin{equation*}
                        \mathrm{softmax}_\beta(x)_i  = \tfrac{1}{\lvert \mathcal{A} \rvert + \sum_{j \notin \mathcal{A}} \exp\left(-\tfrac{M - x_j}{\beta}\right)}.
                    \end{equation*}
                    Therefore
                    \begin{equation}\label{app:eq:soft1}
                        \tfrac{1}{\lvert \mathcal{A} \rvert} - \mathrm{softmax}_\beta(x)_i= \tfrac{\sum_{j \notin \mathcal{A}} \exp\left(-\tfrac{M - x_j}{\beta}\right)}{\lvert \mathcal{A} \rvert \left(\lvert \mathcal{A} \rvert + \sum_{j \notin \mathcal{A}} \exp\left(-\tfrac{M - x_j}{\beta}\right)\right)}.
                    \end{equation}
                    Recall that we defined $\gamma$ as $\gamma(x) := 
                    \begin{cases}
                    \min_{j \notin \mathcal{A}} \bigl(M - x_j\bigr), & \text{if } \mathcal{A} \neq [n] \\
                    +\infty, & \text{if } \mathcal{A} = [n]
                    \end{cases}$.
                    Then, by definition of $\gamma$, for any $j \notin \mathcal{A}$, $M - x_j \geq \gamma(x) > 0$ and $\beta > 0$, we have that $0 \leq \sum_{j \notin \mathcal{A}} \exp\left(-\tfrac{M - x_j}{\beta}\right) \leq  (n - \lvert\mathcal{A}\rvert)e^{-\gamma(x)/\beta}$. 
                    This also means that
                    $0 \leq \tfrac{1}{\lvert \mathcal{A} \rvert \left(\lvert \mathcal{A} \rvert + \sum_{j \notin \mathcal{A}} \exp\left(-\tfrac{M - x_j}{\beta}\right)\right)} \leq \tfrac{1}{\lvert \mathcal{A} \rvert^2}$.
            
                    Finally, substituting these in \eqref{app:eq:soft1}, we obtain
                    \begin{equation*}
                        \left\lvert \mathrm{softmax}_\beta(x)_i - \tfrac{1}{\lvert \mathcal{A} \rvert} \right\rvert \leq  \tfrac{n - \lvert\mathcal{A}\rvert}{\lvert \mathcal{A} \rvert^2} e^{-\gamma(x)/\beta}.
                    \end{equation*}

                    \item If instead $i \notin \mathcal{A}$, the denominator $\lvert \mathcal{A} \rvert + \sum_{j \notin \mathcal{A}} \exp\left(-\tfrac{M - x_j}{\beta}\right)$ is lower bounded by $\lvert \mathcal{A} \rvert$, hence
                    \begin{equation*}
                    0 \leq \mathrm{softmax}_\beta(x)_i \leq \frac{\exp\left(-\tfrac{M - x_i}{\beta}\right)}{\lvert \mathcal{A} \rvert}.
                    \end{equation*}
                    Similar to (i), we note that $M - x_i \geq \gamma(x) > 0$, therefore, since $\beta > 0$
                    \begin{equation*}
                        0 \leq \mathrm{softmax}_\beta(x)_i \leq \tfrac{e^{-\gamma(x)/\beta}}{\lvert \mathcal{A} \rvert} \iff \left\lvert \mathrm{softmax}_\beta(x)_i - 0\right\rvert \leq \tfrac{e^{-\gamma(x)/\beta}}{\lvert \mathcal{A} \rvert}. 
                    \end{equation*}
                \end{enumerate}
                Combining the two cases yields
                \begin{equation*}
                    \left\lvert \mathrm{softmax}_\beta(x)_i - \frac{1}{|\mathcal{A}|}\mathds{1}(i \in \mathcal{A}) \right\rvert \leq \max\left\{\frac{1}{\lvert \mathcal{A} \rvert}, \frac{n - \lvert \mathcal{A} \rvert}{\lvert \mathcal{A} \rvert^2}\right\} e^{-\gamma(x) / \beta}.
                \end{equation*}
            \end{proof}
            
        \subsubsection{Empirical Deviation Bounds}\label{app:empirical_bounds}
            We next control the deviation between the empirical and population-level quantities appearing in the fairness criterion. Lemma~\ref{app:lemma:empirical_bound_fixed_pk} directly follows from similar arguments as in the proof of Lemma~C.2 in \cite{Denis_Elie_Hebiri_Hu_2024}. 
            
            \begin{lemma}\label{app:lemma:empirical_bound_fixed_pk}
                For each $k \in [K]$, let us suppose that 
                $\bar{p}_k : \mathcal{X} \times \mathcal{S} \to [0,1]$ 
                be any mapping such that for all $(x, s) \in \mathcal{X} \times \mathcal{S},\; (\bar{p}_k(x,s))_{k \in [K]}$ lies in the simplex. 
                For each $k \in [K]$, define
                \begin{equation*}
                    \hat{A}_k = \left\lvert \sum_{s \in \mathcal{S}} s \biggl(\mathbb{P}_{X \mid S=s} - \hat{\mathbb{P}}_{X \mid S=s} \biggr)\biggl(\forall j \neq k :\ell^s_k(X; \bar{\lambda}) > \ell^s_j(X; \bar{\lambda})\biggr)\right\rvert.
                \end{equation*}
                Then there exists a constant $C > 0$, independent of any parameter, such that 
                \begin{equation}
                \label{app:eq:bound_Ak_fixed}
                    \mathbb{E}\left[\hat{A}_k \mathds{1}(N_{\min} \geq 1) \; \middle\vert \; D_N \right] \leq \frac{C \mathds{1}(N_{\min} \geq 1)}{\sqrt{N_{\min}}}.
                \end{equation}
            \end{lemma}
            
            \begin{proof}
                Let $k \in [K]$, $\bar{p}_k$ as defined in Section~\ref{sec:method}, and $\hat{A}_k$ as define above. We reason on the event $\{N_{\min}\geq1\}$. 
                Recall that $\mathcal{S} = \{-1, 1\}$, hence for all $s \in \mathcal{S}$, $\lvert s \rvert = 1$. Then by the triangle inequality,
                \begin{equation}\label{app:eq:C7}
                    \hat{A}_k \leq \sum_{s \in \mathcal{S}} \left\lvert \left(\mathbb{P}_{X \mid S=s} - \hat{\mathbb{P}}_{X \mid S=s}\right)\left[\forall j \neq k : \ell^s_k(X; \bar{\lambda}) > \ell^s_j(X; \bar{\lambda})\right]\right\rvert.
                \end{equation}
                
                We now study each term of the sum independently. For each $s \in \mathcal{S}$, by definition of $\ell^s_i$, we have
                \begin{multline*}
                	\left\lvert \left(\mathbb{P}_{X \mid S=s} - \hat{\mathbb{P}}_{X \mid S=s}\right)\left[\forall j \neq k : \ell^s_k(X; \bar{\lambda}) > \ell^s_j(X; \bar{\lambda})\right]\right\rvert\\
                	= \left\lvert \left(\mathbb{P}_{X \mid S=s} - \hat{\mathbb{P}}_{X \mid S=s}\right)\left[\forall j \neq k : \bar{p}_k(X,s) - \bar{p}_j(X,s) > \frac{s \left((\bar{\lambda}^{(1)}_k - \bar{\lambda}^{(2)}_k) - (\bar{\lambda}^{(1)}_j - \bar{\lambda}^{(2)}_j)\right)}{\bar{\pi}_s} \right]\right\rvert 
                \end{multline*}
                where $\bar{\pi}_s$ is defined for each $s \in \mathcal{S}$ in \eqref{eq:pi}. Let us denote $t_{kj} = \frac{s \left((\bar{\lambda}^{(1)}_k - \bar{\lambda}^{(2)}_k) - (\bar{\lambda}^{(1)}_j - \bar{\lambda}^{(2)}_j)\right)}{\bar{\pi}_s}$, then
                \begin{equation*}
                	\left\lvert \left(\mathbb{P}_{X \mid S=s} - \hat{\mathbb{P}}_{X \mid S=s}\right)\left[\forall j \neq k : \ell^s_k(X; \bar{\lambda}) > \ell^s_j(X; \bar{\lambda})\right]\right\rvert
                	= \left\lvert \left(\mathbb{P}_{X \mid S=s} - \hat{\mathbb{P}}_{X \mid S=s}\right)\left[\forall j \neq k : \bar{p}_k(X,s) - \bar{p}_j(X,s) > t_{kj} \right]\right\rvert.
                \end{equation*}
                
                Since
                \begin{equation*}
                	\left\lvert \left(\mathbb{P}_{X \mid S=s} - \hat{\mathbb{P}}_{X \mid S=s}\right)\left[\forall j \neq k : \bar{p}_k(X,s) - \bar{p}_j(X,s) > t_{kj} \right]\right\rvert \leq \sup_{t \in \mathbb{R}} \left\lvert \left(\mathbb{P}_{X \mid S=s} - \hat{\mathbb{P}}_{X \mid S=s}\right)\left[\forall j \neq k : \bar{p}_k(X,s) - \bar{p}_j(X,s) > t\right]\right\rvert,
                \end{equation*}
                
                we obtain
                
                \begin{equation*}
                	\left\lvert \left(\mathbb{P}_{X \mid S=s} - \hat{\mathbb{P}}_{X \mid S=s}\right)\left[\forall j \neq k : \ell^s_k(X; \bar{\lambda}) > \ell^s_j(X; \bar{\lambda})\right]\right\rvert
                	\leq \sup_{t \in \mathbb{R}} \left\lvert \left(\mathbb{P}_{X \mid S=s} - \hat{\mathbb{P}}_{X \mid S=s}\right)\left[\forall j \neq k : \bar{p}_k(X,s) - \bar{p}_j(X,s) > t\right]\right\rvert.
                \end{equation*}

                Observing that $\{\forall j \neq k : \bar{p}_k(X,s) - \bar{p}_j(X,s) > t\} = \{\bar{p}_k(X,s) - \max_{j \neq k} \bar{p}_j(X,s) > t\}$, we get
                \begin{equation*}
                	\left\lvert \left(\mathbb{P}_{X \mid S=s} - \hat{\mathbb{P}}_{X \mid S=s}\right)\left[\forall j \neq k : \ell^s_k(X; \bar{\lambda}) > \ell^s_j(X; \bar{\lambda})\right]\right\rvert
                	\leq \sup_{t \in \mathbb{R}} \left\lvert \left(\mathbb{P}_{X \mid S=s} - \hat{\mathbb{P}}_{X \mid S=s}\right)\left[\bar{p}_k(X,s) - \max_{j \neq k}\bar{p}_j(X,s) > t\right]\right\rvert.
                \end{equation*}
                
                Hence,
                \begin{equation}\label{app:eq:bound_A_k}
                    \hat{A}_k \leq \sum_{s \in \mathcal{S}} \sup_{t \in \mathbb{R}}\left\lvert \left(\mathbb{P}_{X \mid S=s} - \hat{\mathbb{P}}_{X \mid S=s}\right)\left[\bar{p}_k(X,s) - \max_{j \neq k}\bar{p}_j(X,s) > t \right]\right\rvert.
                \end{equation}
                
                For every $u > 0$, the Dvoretzky–Kiefer–Wolfowitz inequality \citep{massart1990dkw} yields
                \begin{equation}\label{app:eq:C10}
                    \mathbb{P}\left[\sup_{t \in \mathbb{R}} \left\lvert(\mathbb{P}_{X \mid S=s} - \hat{\mathbb{P}}_{X \mid S=s})[\bar{p}_k(X,s) - \max_{j \neq k}\bar{p}_j(X,s) > t] \right\rvert > u \;\middle|\; D_N\right] \leq 2 e^{-2 N_s u^2}.
                \end{equation}

                Using the tail–integral formula for non negative random variables,
                $\mathbb{E}[X]=\int_{0}^{\infty}\mathbb{P}(X>u)\,du$, we have
                \begin{multline*}
                    \mathbb{E}\left[\sup_{t \in \mathbb{R}}\left\lvert \left(\mathbb{P}_{X\mid S=s}-\hat{\mathbb{P}}_{X\mid S=s}\right) \left[\bar{p}_k(X,s) - \max_{j \neq k}\bar{p}_j(X,s)>t\right] \right\rvert \;\middle\vert\; D_N\right] \\ =\int_{0}^{\infty}\mathbb{P}\left(
                    \sup_{t \in \mathbb{R}}\left\lvert \left(\mathbb{P}_{X\mid S=s}-\hat{\mathbb{P}}_{X\mid S=s}\right) \left[\bar{p}_k(X,s) - \max_{j \neq k}\bar{p}_j(X,s)>t\right] \right\rvert > u \; \middle\vert \; D_N\right) \, du .
                \end{multline*}
                
                By \eqref{app:eq:C10}, the integrand is bounded by $2e^{-2 N_s u^2}$, hence
                \begin{equation*}                    
                    \mathbb{E}\left[\sup_{t \in \mathbb{R}}\left\lvert \left(\mathbb{P}_{X\mid S=s}-\hat{\mathbb{P}}_{X\mid S=s}\right) \left[\bar{p}_k(X,s) - \max_{j \neq k}\bar{p}_j(X,s)>t\right] \right\rvert \;\middle\vert\; D_N\right] \leq \int_{0}^{\infty} 2e^{-2N_su^2}\,du .
                \end{equation*}                    
                Since $2\int_{0}^{\infty} e^{-a x^{2}}\,dx=\sqrt{\pi/a}$, we conclude that
                \begin{equation}\label{app:eq:C11}                    
                    \mathbb{E}\left[\sup_{t \in \mathbb{R}}\left\lvert \left(\mathbb{P}_{X\mid S=s}-\hat{\mathbb{P}}_{X\mid S=s}\right) \left[\bar{p}_k(X,s) - \max_{j \neq k}\bar{p}_j(X,s)>t\right] \right\rvert \;\middle\vert\; D_N\right] \leq \sqrt{\frac{\pi}{2N_s}} .
                \end{equation}

                Combining \eqref{app:eq:bound_A_k} and \eqref{app:eq:C11}, we obtain
                \begin{equation}\label{app:eq:C12}
                    \mathbb{E}\left[\hat{A}_k \mid D_N\right] \leq \sum_{s \in \mathcal{S}} \sqrt{\frac{\pi}{2 N_s}} \leq \sqrt{\frac{2 \pi}{N_{\min}}}.
                \end{equation}
                Setting $C = \sqrt{2\pi}$ yields the desired bound
                \begin{equation}
                    \mathbb{E}\left[\hat{A}_k \mathds{1}(N_{\min} \geq 1) \; \middle\vert \; D_N \right] \leq \frac{C \mathds{1}(N_{\min} \geq 1)}{\sqrt{N_{\min}}}.
                \end{equation}
            \end{proof}
            
        \subsubsection{Private Optimization}
            Finally, we analyze the optimization error induced by running the differentially private stochastic gradient descent (DP-SGD) algorithm on the smoothed fairness objective.

            Lemma~\ref{app:lemma:bound_dp} provides a quantitative guarantee on the expected gradient norm of the privatized objective, which directly controls the stability of the fairness constraint under DP optimization noise.
            
            \begin{lemma}\label{app:lemma:bound_dp}
                Let $(\bar{\lambda}^{(1)}, \bar{\lambda}^{(2)})$ be the parameters output by DP2DP. Denote by $\mathcal{B}$ the random mini\textnormal{-}batch sequence and by $\mathbf{Z}=(Z_1,\dots,Z_T)$ the Gaussian noises added at each step. Then,  for $\hat{H}_{\beta}$ defined in \eqref{eq:H-beta}, we have
                \begin{equation*}
                    \mathbb{E}_{\mathcal{B},\mathbf{Z}}\left[\left\lVert \nabla \hat{H}_{\beta}(\bar{\lambda}^{(1)}, \bar{\lambda}^{(2)}) \right\rVert_\infty\right] \leq 4\sqrt{\tfrac{1}{\beta}\tfrac{C_\lambda \log T}{\sqrt{T}}(2\sqrt{2} +\rho\sqrt{2K} + \tfrac{2K\sigma_{\textsc{sgd}}}{b})}.
                \end{equation*}
                where the expectation is taken over the algorithmic randomness $(\mathcal{B},\mathbf{Z})$.
            \end{lemma}
            \begin{proof}
                Let $(x, s) \in \mathcal{X} \times \mathcal{S}$. Lemma~\ref{app:lemma:sample_loss} show that the per-sample $\hat{h}(\cdot; x, s)$ loss is $\tfrac{8}{\beta}$-smooth with respect to the $\ell_\infty$-norm, thus $\hat{H}_\beta(\lambda^{(1)}, \lambda^{(2)}) := \tfrac{1}{\lvert \mathcal{S} \rvert}\sum_{s \in \mathcal{S}} \tfrac{1}{N_s}\sum_{x \in D_{N \vert \mathcal{X}, s}} \hat{h}(\lambda; x,s)$ is also $\tfrac{8}{\beta}$-smooth with respect to the $\ell_\infty$-norm by triangle inequality and homogeneity of the norm. 
                Applying Lemma~\ref{app:lemma:ineq-smooth}, for all $(\lambda^{(1)}, \lambda^{(2)}) \in \mathbb{R}^{2K}$ one has the inequality
                \begin{equation*}
                    \left\lVert \nabla \hat{H}_{\beta}(\lambda^{(1)}, \lambda^{(2)}) \right\rVert^2_1 \leq \tfrac{16}{\beta}\left(\hat{H}_{\beta}(\lambda^{(1)}, \lambda^{(2)}) - \hat{H}_{\beta}(\hat{\lambda}^{(1)}, \hat{\lambda}^{(2)}) \right),
                \end{equation*}
                where $(\hat{\lambda}^{(1)}, \hat{\lambda}^{(2)}) \in \underset{(\lambda^{(1)}, \lambda^{(2)}) \in \mathbb{R}_+^{2K}}{\arg\min} \hat{H}_\beta(\lambda^{(1)}, \lambda^{(2)})$.
                
                In particular, considering $(\bar{\lambda}^{(1)}, \bar{\lambda}^{(2)})$ as in Algorithm~1, we have
                \begin{equation*}
                    \left\lVert \nabla \hat{H}_{\beta}(\bar{\lambda}^{(1)}, \bar{\lambda}^{(2)}) \right\rVert^2_1 \leq \tfrac{16}{\beta}\left(\hat{H}_{\beta}(\bar{\lambda}^{(1)}, \bar{\lambda}^{(2)}) - \hat{H}_{\beta}(\hat{\lambda}^{(1)}, \hat{\lambda}^{(2)}) \right).
                \end{equation*}

                Using Lemma~\ref{app:lemma:sample_loss} we obtain the Lipschitz constant $L = 2\sqrt{2} + \rho \sqrt{2K}$, and Lemma~\ref{app:lemma:minimum} further implies that the diameter of the $\lambda$-space is $C_\lambda \sqrt{2K}$.
                Taking expectation with respect to the randomness of the minibatch and the Gaussian noise, and using Theorem~\ref{app:lemma:utility_dpsgd}, yields
                \begin{equation}\label{app:eq:bound-norm1}
                    \mathbb{E}_{\mathcal{B},\mathbf{Z}}\left[\left\lVert \nabla \hat{H}_{\beta}(\bar{\lambda}^{(1)}, \bar{\lambda}^{(2)}) \right\rVert^2_1\right] \leq \tfrac{16}{\beta} \tfrac{C_\lambda \log T}{\sqrt{T}}(2\sqrt{2} + \rho \sqrt{2K} + \tfrac{\sigma_{\textsc{sgd}}\sqrt{2K}}{b}).
                \end{equation}
                
                Applying Jensen’s inequality to \eqref{app:eq:bound-norm1} we obtain
                \begin{equation*}
                    \mathbb{E}_{\mathcal{B},\mathbf{Z}}\left[\left\lVert \nabla \hat{H}_{\beta}(\bar{\lambda}^{(1)}, \bar{\lambda}^{(2)}) \right\rVert_1\right] \leq \sqrt{\mathbb{E}_{\mathcal{B},\mathbf{Z}}\left[\left\lVert \nabla \hat{H}_{\beta}(\bar{\lambda}^{(1)}, \bar{\lambda}^{(2)}) \right\rVert^2_1\right]} \leq  4\sqrt{\tfrac{C_\lambda \log T}{\beta \sqrt{T}}(2\sqrt{2} + \rho \sqrt{2K} + \tfrac{\sigma_{\textsc{sgd}}\sqrt{2K}}{b})}.
                \end{equation*}
                
                Finally, since
                \begin{equation*}
                    \left\lVert \nabla \hat{H}_{\beta}(\bar{\lambda}^{(1)}, \bar{\lambda}^{(2)}) \right\rVert_\infty \leq \left\lVert \nabla \hat{H}_{\beta}(\bar{\lambda}^{(1)}, \bar{\lambda}^{(2)}) \right\rVert_1,
                \end{equation*}
                we get the expected result
                \begin{equation*}
                    \mathbb{E}_{\mathcal{B},\mathbf{Z}}\left[\left\lVert \nabla \hat{H}_{\beta}(\bar{\lambda}^{(1)}, \bar{\lambda}^{(2)}) \right\rVert_\infty\right] 
                    \leq 4\sqrt{\tfrac{C_\lambda \log T}{\beta \sqrt{T}}(2\sqrt{2} + \rho \sqrt{2K} + \tfrac{\sigma_{\textsc{sgd}}\sqrt{2K}}{b})}.
                \end{equation*}
            \end{proof}
            
    \subsection{Proof of Theorem~\ref{thm:fairness}}
        \begin{lemma}[Lemma C.1~\citep{Denis_Elie_Hebiri_Hu_2024}]\label{app:lemma:proba_multi_max}
            With respect to the data $D_N$, we have that, for all $\lambda \in \mathbb{R}^{2K}_+$, for each $s \in \mathcal{S}$ and $k \in [K]$,
            \begin{equation*}
                \hat{\mathbb{P}}_{X \mid S=s}[\exists j \neq k : \ell^s_k(X; \lambda) := \ell^s_j(X; \lambda)] := \frac{1}{N_s} \sum_{X \in D_{N \mid \mathcal{X}, s}} \mathds{1}(\exists j \neq k : \ell^s_k(X_i; \lambda) = \ell^s_j(X_i; \lambda)) \leq \cfrac{K-1}{N_{s}} \; \text{a.s.}
            \end{equation*}
        \end{lemma}
        We now combine the previous lemmas to establish the main fairness bound.
        
        \begin{theorem}\label{app:theo_fair}
            Consider \textsc{DP2DP} scheme, as defined in Algorithm~\ref{algo:dp-fair} with a step size $\eta_t$ define for all $t \in [T]$ in Lemma~\ref{app:lemma:utility_dpsgd}. Let us also denote by $\pi_{\min} = \min\{ \mathbb{P}\left[ S = s \right] \mid s \in \mathcal{S} \}$ the minimum group size within sensitive attributes w.r.t. $\mathbb{P}$. Then, there exist constants $C_1 > 0$ depending on $K$ and $\pi_{\min}$, $\gamma > 0$, and  $C_2 > 0$ that depends on $K$, and $\gamma$, such that for any conditional probabilities $\bar{p}_k$ computed in Phase 1, one has
            \begin{equation*}
                \mathbb{E}_{D_N, \mathcal{B},\mathbf{Z}}\!\left[\mathcal{U}\left(\hat{g}_\rho \right)\right] \leq \rho + \frac{C_1}{\sqrt{N}} + C_2 e^{-\tfrac{\gamma}{\beta}} + 4\sqrt{\tfrac{ \sqrt{2K} C_\lambda \log T}{\beta \sqrt{T}}\left(2\sqrt{2} + \rho \sqrt{2K} + \tfrac{\sigma_{\textsc{sgd}}\sqrt{2K}}{b}\right)}.
            \end{equation*}
            In the above, the expectation is taken over the sampling of $D_N$ and the randomness of the algorithm (mini-batch sampling and Gaussian noise).
        \end{theorem}
        \begin{proof}
                We begin by recalling that the decision rule $\hat{g}_\rho$ is defined for all $(x, s) \in \mathcal{X} \times \mathcal{S}$ by
                \begin{equation*}
                    \hat{g}_{\rho}(x, s) = \underset{k \in [K]}{\arg\max} \left\{\bar{\pi}_s \bar{p}_k(x, s) - s(\bar{\lambda}^{(1)}_k - \bar{\lambda}^{(2)}_k) \right\},
                \end{equation*}
                where $(\bar{\lambda}^{(1)}, \bar{\lambda}^{(2)})$ are the parameters obtained by running DP-SGD on $\hat{H}_{\beta}$.
                
                The strategy of the proof is to show that the calibration of the Lagrange parameters leads directly to the announced upper bound on the unfairness.
                
                
                For each $k\in[K]$ and $l\in\{1,2\}$, set
                \begin{equation*}
                    \bar\eta^{(l)}_{k,\beta} :=\nabla_{\lambda^{(l)}_k}\hat{H}_\beta\left(\bar\lambda^{(1)}, \bar\lambda^{(2)}\right).
                \end{equation*}
                
                A direct differentiation gives
                \begin{align*}
                    \nabla_{\lambda^{(l)}_k}\hat{H}_\beta\left(\bar\lambda^{(1)},\bar\lambda^{(2)}\right)
                    &=\sum_{s\in\mathcal S}\hat{\mathbb{E}}_{X\mid S=s}\left[\nabla_{\lambda^{(l)}_k}\,\mathrm{LSE}_\beta\!\left(\ell^s(X;\bar\lambda)\right)\right]+\rho\\
                    &=(2l-3)\sum_{s\in\mathcal S} s \hat{\mathbb{E}}_{X\mid S=s}\!\left[\mathrm{softmax}_\beta \left(\ell^s(X;\bar\lambda)\right)_k\right]+\rho,
                \end{align*}
                where
                \begin{equation*}
                    \mathrm{softmax}_\beta \left(\ell^s(X;\lambda)\right)_k := \frac{\exp\{\ell^s_k(X;\lambda)/\beta\}}{\sum_{j=1}^K \exp\{\ell^s_j(X;\lambda)/\beta\}}.
                \end{equation*}
                
                Decomposing according to whether class $k$ is the unique maximizer or ties with others,
                \begin{multline*}
                    \bar\eta^{(l)}_{k,\beta}=(2l-3)\sum_{s\in\mathcal S}s\Biggl(\hat{\mathbb{E}}_{X\mid S=s}\left[\mathrm{softmax}_\beta\bigl(\ell^s(X;\bar\lambda)\bigr)_k \mathds 1\left(\forall j\neq k: \ell^s_k(X; \bar{\lambda})>\ell^s_j(X; \bar{\lambda})\right)\right] \\
                    + \hat{\mathbb{E}}_{X\mid S=s}\left[\mathrm{softmax}_\beta\bigl(\ell^s(X;\bar\lambda)\bigr)_k \mathds 1\left(\forall j \neq k: \ell^s_k(X; \bar{\lambda}) \geq \ell^s_j(X; \bar{\lambda}),\ \exists j\neq k: \ell^s_k(X; \bar{\lambda})=\ell^s_j(X; \bar{\lambda})\right)\right]\Biggr)+\rho .
                \end{multline*}
                
                By Lemma~\ref{app:lemma:softmax}, there exist constants $C_2^s>0$ (depending on $K$ and $s$), $\gamma^s>0$, and numbers $u_k^s\in[0,1]$ such that
                \begin{multline*}
                    \bar\eta^{(l)}_{k,\beta}\leq (2l-3)\sum_{s\in\mathcal S}s\Biggl[\hat{\mathbb{P}}_{X\mid S=s}\left(\forall j\neq k: \ell^s_k(X; \bar{\lambda})>\ell^s_j(X; \bar{\lambda})\right)\\
                    + u_k^s \hat{\mathbb{P}}_{X\mid S=s}\left(\forall j \neq k:\ \ell^s_k(X; \bar{\lambda}) \geq \ell^s_j(X; \bar{\lambda}), \exists j\neq k: \ell^s_k(X; \bar{\lambda}) = \ell^s_j(X; \bar{\lambda})\right) + C_2^s e^{-\gamma^s/\beta}\Biggr]+\rho .
                \end{multline*}

                Since the event $\{\forall j\neq k: \ell^s_k(X; \bar{\lambda}) \geq \ell^s_j(X; \bar{\lambda}), \exists j \neq k: \ell^s_k(X; \bar{\lambda}) = \ell^s_j(X; \bar{\lambda})\}$ is contained in $\{\exists j \neq k: \ell^s_k(X; \bar{\lambda}) = \ell^s_j(X; \bar{\lambda})\}$, we may apply
                Lemma~\ref{app:lemma:proba_multi_max}. It follows that the second probability on the
                right-hand side is bounded almost surely with respect to $D_N$ by
                $\tfrac{K-1}{N_s}$, and therefore by $\tfrac{K-1}{N_{\min}}$.
                Hence, for each $s\in\mathcal S$,
                \begin{equation*}
                    0 \leq \hat{\mathbb{P}}_{X\mid S=s}\left(\forall j \neq k: \ell^s_k(X; \bar{\lambda}) \geq \ell^s_j(X; \bar{\lambda}), \exists j \neq k: \ell^s_k(X; \bar{\lambda}) = \ell^s_j(X; \bar{\lambda})\right) \leq \frac{K-1}{N_{\min}}.
                \end{equation*}

                Consequently, for each $s\in\mathcal S$,
                \begin{equation*}
                    0 \leq u_k^s \hat{\mathbb{P}}_{X\mid S=s}\left(\forall j\neq k: \ell^s_k(X; \bar{\lambda}) \geq \ell^s_j(X; \bar{\lambda}), \exists j\neq k: \ell^s_k(X; \bar{\lambda})=\ell^s_j(X; \bar{\lambda})\right) \leq \frac{K-1}{N_{\min}}.
                \end{equation*}
                
                Taking the difference with the signs $s\in\{-1,1\}$ yields the symmetric bound
                \begin{equation*}
                    -\frac{K-1}{N_{\min}} \leq \sum_{s\in\mathcal S} s u_k^s \hat{\mathbb{P}}_{X\mid S=s}\left(\forall j\neq k: \ell^s_k(X; \bar{\lambda}) \geq \ell^s_j(X; \bar{\lambda}), \exists j \neq k: \ell^s_k(X; \bar{\lambda})=\ell^s_j(X; \bar{\lambda})\right) \leq \frac{K-1}{N_{\min}}.
                \end{equation*}

                Hence
                \begin{equation*}
                    \bar{\eta}^{(l)}_{k, \beta} \leq (2l-3) \sum_{s \in \mathcal{S}} s \hat{\mathbb{P}}_{X \mid S=s} \Bigl(\forall j \neq k : \ell^s_k(X; \bar{\lambda}) > \ell^s_j(X; \bar{\lambda}) \Bigr)
                    + \frac{K-1}{N_{\min}} + (2l-3) \sum_{s \in \mathcal{S}}s C_2^s e^{-\frac{\gamma^s}{\beta}} + \rho.
                \end{equation*}
                
                Specializing to $l=1$ and $l=2$, this yields the inequalities
                \begin{align*}
                    \bar\eta^{(1)}_{k,\beta} &\leq -\sum_{s \in \mathcal{S}} s \hat{\mathbb{P}}_{X \mid S=s}\Bigl(\forall j \neq k : \ell^s_k(X; \bar{\lambda}) > \ell^s_j(X; \bar{\lambda})\Bigr)  + \frac{K-1}{N_{\min}} - \sum_{s \in \mathcal{S}}s C_2^s e^{-\frac{\gamma^s}{\beta}} + \rho, \\
                    \bar\eta^{(2)}_{k,\beta} &\leq \sum_{s \in \mathcal{S}} s \hat{\mathbb{P}}_{X \mid S=s}\Bigl(\forall j \neq k : \ell^s_k(X; \bar{\lambda}) > \ell^s_j(X; \bar{\lambda})\Bigr) + \frac{K-1}{N_{\min}} + \sum_{s \in \mathcal{S}}s C_2^s e^{-\frac{\gamma^s}{\beta}} + \rho.
                \end{align*}
    
                Therefore
                \begin{align*}
                     \sum_{s \in \mathcal{S}} s \hat{\mathbb{P}}_{X \mid S=s}\Bigl(\forall j \neq k : \ell^s_k(X; \bar{\lambda}) > \ell^s_j(X; \bar{\lambda})\Bigr) &\leq -\bar\eta^{(1)}_{k,\beta}  + \frac{K-1}{N_{\min}} - \sum_{s \in \mathcal{S}}s C_2^s e^{-\frac{\gamma^s}{\beta}} + \rho, \\
                    \sum_{s \in \mathcal{S}} s \hat{\mathbb{P}}_{X \mid S=s}\Bigl(\forall j \neq k : \ell^s_k(X; \bar{\lambda}) > \ell^s_j(X; \bar{\lambda})\Bigr) &\geq \bar\eta^{(2)}_{k,\beta} - \frac{K-1}{N_{\min}} - \sum_{s \in \mathcal{S}}s C_2^s e^{-\frac{\gamma^s}{\beta}} - \rho.
                \end{align*}
                
                Combining both gives the interval bound
                \begin{equation}\label{app:eq:bound_empi_prob}
                    \Biggl\lvert \sum_{s \in \mathcal{S}} s \hat{\mathbb{P}}_{X \mid S=s} \Bigl(\forall j \neq k : \ell^s_k(X; \bar{\lambda}) > \ell^s_j(X; \bar{\lambda})\Bigr) \Biggr\rvert 
                    \leq \max\bigl\{|\bar\eta^{(1)}_{k,\beta}|, |\bar\eta^{(2)}_{k,\beta}|\bigr\}  + \frac{K-1}{N_{\min}} + 2C_2 e^{-\frac{\gamma}{\beta}} + \rho.
                \end{equation}
                where $C_2 = \max\{C_2^1, C_2^{-1}\}$ and $\gamma = \min\{\gamma^1, \gamma^{-1}\}$.            
                As in \cite{Denis_Elie_Hebiri_Hu_2024}, we turn to the unfairness of $\hat{g}_\rho$:
                \begin{equation*}
                    \left\lvert \sum_{s \in \mathcal{S}} s \mathbb{P}_{X \mid S=s}\bigl(\hat{g}_\rho(X,S) = k \bigr) \right\rvert 
                    = \left\lvert \sum_{s \in \mathcal{S}} s \mathbb{P}_{X \mid S=s}\Bigl(\forall j \neq k : \ell^s_k(X; \bar{\lambda}) > \ell^s_j(X; \bar{\lambda})\Bigr)\right\rvert \\
                \end{equation*}
                By using the triangle inequality, and the random variable $\hat{A}_k$ introduce in Lemma~\ref{app:lemma:empirical_bound_fixed_pk},
                \begin{equation*}
                    \left\lvert \sum_{s \in \mathcal{S}} s \mathbb{P}_{X \mid S=s}\bigl(\hat{g}_\rho(X,S) = k \bigr) \right\rvert \leq \hat{A}_k + \left\lvert \sum_{s \in \mathcal{S}} s \hat{\mathbb{P}}_{X \mid S=s}\Bigl(\forall j \neq k : \ell^s_k(X; \bar{\lambda}) > \ell^s_j(X; \bar{\lambda})\Bigr)\right\rvert
                \end{equation*}
                With \eqref{app:eq:bound_empi_prob}, we get
                \begin{equation*}
                    \left\lvert \sum_{s \in \mathcal{S}} s \mathbb{P}_{X \mid S=s}\bigl(\hat{g}_\rho(X,S) = k \bigr) \right\rvert \leq \hat{A}_k+ \rho + \max_{l \in \{1,2\}}|\hat\eta^{(l)}_{k,\beta}|  + \tfrac{K-1}{N_{\min}} + 2C_2 e^{-\frac{\gamma}{\beta}}.
                \end{equation*}
                Maximizing over $k \in [K]$ gives
                \begin{multline*}
                    \mathcal{U}(\hat{g}_\rho) := \max_{k \in [K]} \left\lvert \sum_{s \in \mathcal{S}} s \mathbb{P}_{X \mid S=s}\bigl(\hat{g}_\rho(X,S) = k \bigr) \right\rvert \leq \max_{k \in [K]} \hat{A}_k + \rho + \lVert \nabla \hat{H}_{\beta}(\bar{\lambda}^{(1)}, \bar{\lambda}^{(2)}) \rVert_\infty  + \tfrac{K-1}{N_{\min}} + 2C_2 e^{-\frac{\gamma}{\beta}}.
                \end{multline*}
    
                Applying Lemma~\ref{app:lemma:empirical_bound_fixed_pk} we deduce that conditional on $D_N$,
                \begin{align*}
                    \mathbb{E}\left[\mathcal{U}(\hat{g}_\rho) \; \middle\vert \; D_N \right] &= \mathbb{E}\left[\mathcal{U}(\hat{g}_\rho)\mathds{1}(N_{\min} \geq 1) \; \middle\vert \; D_N \right] + \mathbb{E}\left[\mathcal{U}(\hat{g}_\rho)\mathds{1}(N_{\min} = 0) \; \middle\vert \; D_N \right] \\
                    &\leq \rho + (\tfrac{C}{\sqrt{N_{\min}}} + \tfrac{K-1}{N_{\min}} )\mathds{1}(N_{\min} \geq 1) + 2 C_2 e^{-\frac{\gamma}{\beta}} \\ &+ \mathbb{E}\left[\lVert \nabla \hat{H}_{\beta}(\bar{\lambda}^{(1)}, \bar{\lambda}^{(2)}) \rVert_\infty \; \middle\vert \; D_N \right] + \mathbb{E}\left[\mathcal{U}(\hat{g}_\rho)\mathds{1}(N_{\min} = 0) \; \middle\vert \; D_N \right]
                \end{align*}
    
                On the event $\{N_{\min} \geq 1\}$, $\tfrac{1}{\sqrt{N_{\min}}} \geq \tfrac{1}{N_{\min}}$, thus there exists a non negative constant $c_1$ such that
                \begin{multline*}
                    \mathbb{E}\left[\mathcal{U}(\hat{g}_\rho) \; \middle\vert \; D_N \right] \leq \rho + \tfrac{c_1}{\sqrt{N_{\min}}}\mathds{1}(N_{\min} \geq 1) + 2 C_2 e^{-\frac{\gamma}{\beta}} \\+ \mathbb{E}\left[\lVert \nabla \hat{H}_{\beta}(\bar{\lambda}^{(1)}, \bar{\lambda}^{(2)}) \rVert_\infty \; \middle\vert \; D_N \right] + \mathbb{E}\left[\mathcal{U}(\hat{g}_\rho)\mathds{1}(N_{\min} = 0) \; \middle\vert \; D_N \right]
                \end{multline*}
                
                Now since all the terms of the r.h.s. unless $\lVert \nabla \hat{H}_{\beta}(\bar{\lambda}^{(1)}, \bar{\lambda}^{(2)}) \rVert_\infty$ and $\mathcal{U}(\hat{g}_\rho)$ are independent from the minibatch sampling and the Gaussian noise vectors of DP-SGD, Lemma~\ref{app:lemma:bound_dp} gives
                \begin{multline*}
                    \mathbb{E}_{\mathcal{B},\mathbf{Z}}\left[\mathbb{E}\left[\mathcal{U}(\hat{g}_\rho) \; \middle\vert \; D_N \right]\right] \leq \rho + \tfrac{c_1 \mathds{1}(N_{\min} \geq 1)}{\sqrt{N_{\min}}} +  \mathbb{E}_{\mathcal{B},\mathbf{Z}}\left[\mathbb{E}\left[\mathcal{U}(\hat{g}_\rho)\mathds{1}(N_{\min} = 0) \; \middle\vert \; D_N \right]\right] \\ + 2 C_2 e^{-\frac{\gamma}{\beta}} + 4\sqrt{\tfrac{C_\lambda \log T}{\beta \sqrt{T}}(2\sqrt{2} + \rho \sqrt{2K} + \tfrac{2K\sigma_{\textsc{sgd}}}{b})}.
                \end{multline*}

                Taking the expectation over $D_N$ and by definition of the conditional expectation there exists a non negative constant $C_K$ that depend on K such that
                \begin{equation*}
                    \mathbb{E}_{D_N, \mathcal{B},\mathbf{Z}}\left[\mathcal{U}(\hat{g}_\rho) \right] \leq \rho + \mathbb{E}\left[\tfrac{c_1 \mathds{1}(N_{\min} \geq 1)}{\sqrt{N_{\min}}}\right] + C_K \mathbb{P}(N_{\min} = 0) + 2 C_2 e^{-\frac{\gamma}{\beta}} + 4\sqrt{\tfrac{C_\lambda \log T}{\beta \sqrt{T}}(2\sqrt{2} + \rho \sqrt{2K} + \tfrac{2K\sigma_{\textsc{sgd}}}{b})}.
                \end{equation*}
                
                Finally, since $\mathbb{P}(N_{\min} = 0) \leq (1-\pi_{-1})^N + (1-\pi_{1})^N$, there exists a constant $C_1$ that depend on $\pi_{\min}$ such that $\mathbb{P}(N_{\min} = 0) \leq \tfrac{C_1}{\sqrt{N}}$, and by applying Lemma~\ref{app:lemma:binom} with Jensen inequality, we obtain
                \begin{equation*}
                    \mathbb{E}_{D_N,\mathcal{B},\mathbf{Z}}\bigl[\mathcal{U}(\hat{g}_\rho) \bigr] \leq \rho + \tfrac{C_1}{\sqrt{N}} + 2 C_2 e^{-\frac{\gamma}{\beta}} + 4\sqrt{\tfrac{C_\lambda \log T}{\beta \sqrt{T}}(2\sqrt{2} + \rho \sqrt{2K} + \tfrac{2K\sigma_{\textsc{sgd}}}{b})}.
                \end{equation*}
                This concludes the proof.
            \end{proof}     
        \subsection{Proof of Corollary 5.1}
            \begin{corollary}\label{app:cor:fairness}
                Let $\rho \ge 0$, and fix $T=N^2$ and $\beta = 2\gamma/\log N$. Then there exists a constant $C_*$, depending on $K$ ,$\pi_{\min}$, $C_\lambda$ ,$ b$ ,$\rho$, and $\sigma_{\textsc{sgd}}$ such that
                \begin{equation*}
                    \mathbb{E}\!\left[\mathcal{U}(\hat{g}_\rho)\right] \leq \rho + C_* \frac{\log N}{\sqrt{N}}.
                \end{equation*}
                As before, the expectation is taken jointly over the sampling of $D_N$ and the randomness of the algorithm, including mini-batch selection and Gaussian noise.
            \end{corollary}
            \begin{proof}
                This result follows immediately from Theorem~\ref{app:theo_fair} by substituting the specific choices $T = N^2$ and $\beta = 2\gamma / \log N$. Under these parameters, the dominant terms in Theorem~\ref{app:theo_fair} scale as $\mathcal{O}(\tfrac{\log N}{\sqrt{N}})$, while the exponential term $e^{-\gamma / \beta}$ becomes negligible when $N \geq e$. All constants are aggregated into $C_*$, which depends on $K$, $\pi_{\min}$, $C_\lambda$, $b$, $\rho$, and $\sigma_{\textsc{sgd}}$.
            \end{proof}
\newpage
\section{ADDITIONAL EXPERIMENTAL DETAILS AND RESULTS}\label{app:exp}
    \subsection{Synthetic Data Generation}
        We describe here the procedure used to generate the synthetic dataset employed in Section~\ref{sec:expe}.  
        
        We define the synthetic data as $(X, S, Y)$. For all $k \in [K]$, we assume a uniform class distribution: $P(Y = k) = 1/K$.  
        Given $Y = k$, the features $X \in \mathbb{R}^d$ follow a Gaussian mixture model with $m$ components:
        
        \begin{equation*}
            (X \mid Y = k) \sim \frac{1}{m} \sum_{i=1}^m \mathcal{N}_d(c^k + \mu^k_i, I_d)
        \end{equation*}
        
        where $c^k \sim \mathcal{U}_d(-1,1)$ and $\mu^k_1, \ldots, \mu^k_m \sim \mathcal{N}_d(0, I_d)$.  
        The sensitive attribute $S \in \{-1, +1\}$ is generated via a Bernoulli contamination process that depends on the class label $k$:
        
        \begin{equation*}
            (S \mid Y = k) \sim 
            \begin{cases}
                2\mathcal{B}(p) - 1 & \text{if } k \leq \lfloor K/2 \rfloor \\
                2\mathcal{B}(1 - p) - 1 & \text{if } k > \lfloor K/2 \rfloor
            \end{cases}
        \end{equation*}
    \subsection{Experimental Setup}
        This section details the experimental configuration used to produce the results in Section~\ref{sec:expe} and Appendix~\ref{app:exp}. We describe privacy parameters, and implementation details, and we summarize the key hyperparameters per dataset in Table~\ref{tab:exp-setup}.

        \paragraph{Privacy Parameters} 
        The standard deviations $\sigma_{\mathrm{\textsc{SGD}}}$ used to achieve a target $(\varepsilon,\delta)$-DP guarantee vary across datasets, as they depend on both the sample size $N$ and the number of optimization iterations $T$, whereas $\sigma_\pi$ is fixed. For each experiment, we calibrate these noise levels individually to satisfy $(\varepsilon, \delta = 10^{-5})$ using the Rényi DP accountant from the \texttt{dp-accounting} library \citep{google_dp_accounting}. The resulting $\varepsilon$ values reported in the figures correspond to the exact privacy guarantees derived from this accounting procedure.

        \paragraph{Implementation and Hardware}
        All experiments are implemented in \texttt{Python} using \texttt{scikit-learn}, \texttt{PyTorch}, and \texttt{joblib} for parallelization. Experiments were run on a \textbf{MacBook Pro 13-inch (M1, 2020), 8\,GB RAM} without GPU acceleration.  
        Each configuration is repeated $15$ times with different random seeds, and results are reported as mean $\pm$ standard deviation. For all real datasets, baseline results are retrieved directly from the public repository of \citet{lowy2023stochastic} to ensure faithful comparison.

        Table~\ref{tab:exp-setup} summary of the experiments for each dataset.
        \begin{table}[ht]
            \centering
            \renewcommand{\arraystretch}{1.2}
            \footnotesize
            \begin{tabular}{lcccc}
                \toprule
                \textbf{Parameter} & \textbf{Synthetic} & \textbf{Adult} & \textbf{Credit Card} & \textbf{Parkinson's} \\
                \midrule
                Model type & Logistic Regression & Logistic Regression & Logistic Regression & Logistic Regression \\
                Model regularization & $\ell_2$, $C=1.0$ & $\ell_2$, $C=1.0$ & $\ell_2$, $C=1.0$ & $\ell_2$, $C=1.0$ \\
                DP2DP iter & $100$ & $100$ & $1000$ & $1000$ \\
                Batch size & $128$ & $128$ & $128$ & $128$ \\
                Parameter $\beta$ & $10^{-5}$ & $10^{-5}$ & $10^{-5}$ & $10^{-5}$ \\
                Fairness $\rho_{\mathrm{fair}}$ & $[0.0, 0.2]$ & $[0.0, 0.2]$ & $[0.0, 0.025]$ & $[0.0, 0.08]$ \\
                Train/Pool/Test split & 60/20/20 & 55/20/25 & 55/20/25 & 55/20/25 \\ 
                \bottomrule
            \end{tabular}
            \caption{Hyperparameter configuration per dataset.}
            \label{tab:exp-setup}
        \end{table}
        
    \subsection{Additional Experimental Results}
        In this section, we present complementary empirical results to those reported in Section~\ref{sec:expe}, including extended comparisons on real-world benchmarks. 
        
        \subsubsection{Results on Default-CCC dataset}
            Figure~\ref{fig:diagram_default} reports the trade-off between demographic parity violation and misclassification error achieved by our method compared to existing approaches, under $(\varepsilon,\delta)$-differential privacy with $\delta = 10^{-5}$. We observe that \textsc{DP2DP} consistently achieves improved fairness-accuracy trade-offs, especially in the low-privacy regime ($\varepsilon = 0.5$), while remaining competitive in terms of classification performance.

            \begin{figure}[ht]
                \begin{subfigure}[b]{0.49\columnwidth}
                    \centering
                    \includegraphics[width=1\textwidth]{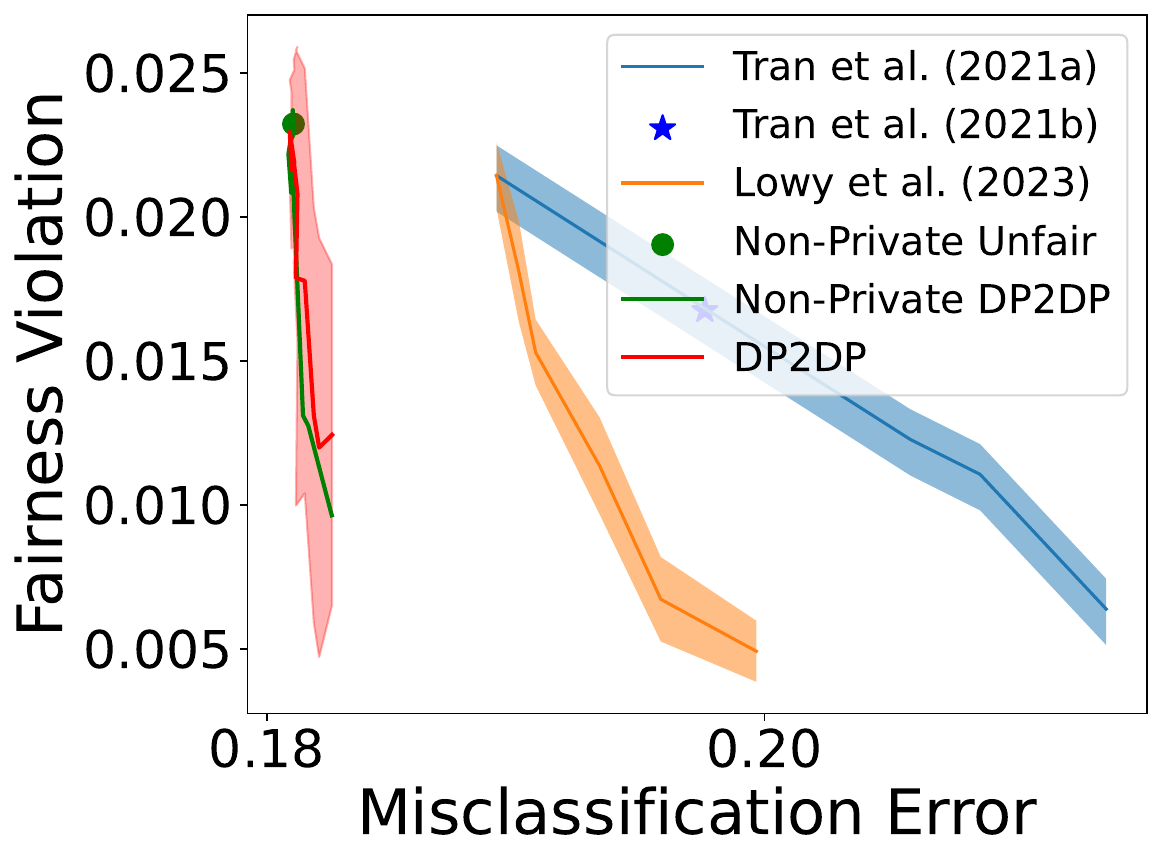}
                    \label{fig:diag_credit_eps05}
                \end{subfigure}
                \begin{subfigure}[b]{0.49\columnwidth}
                    \centering
                    \includegraphics[width=1\textwidth]{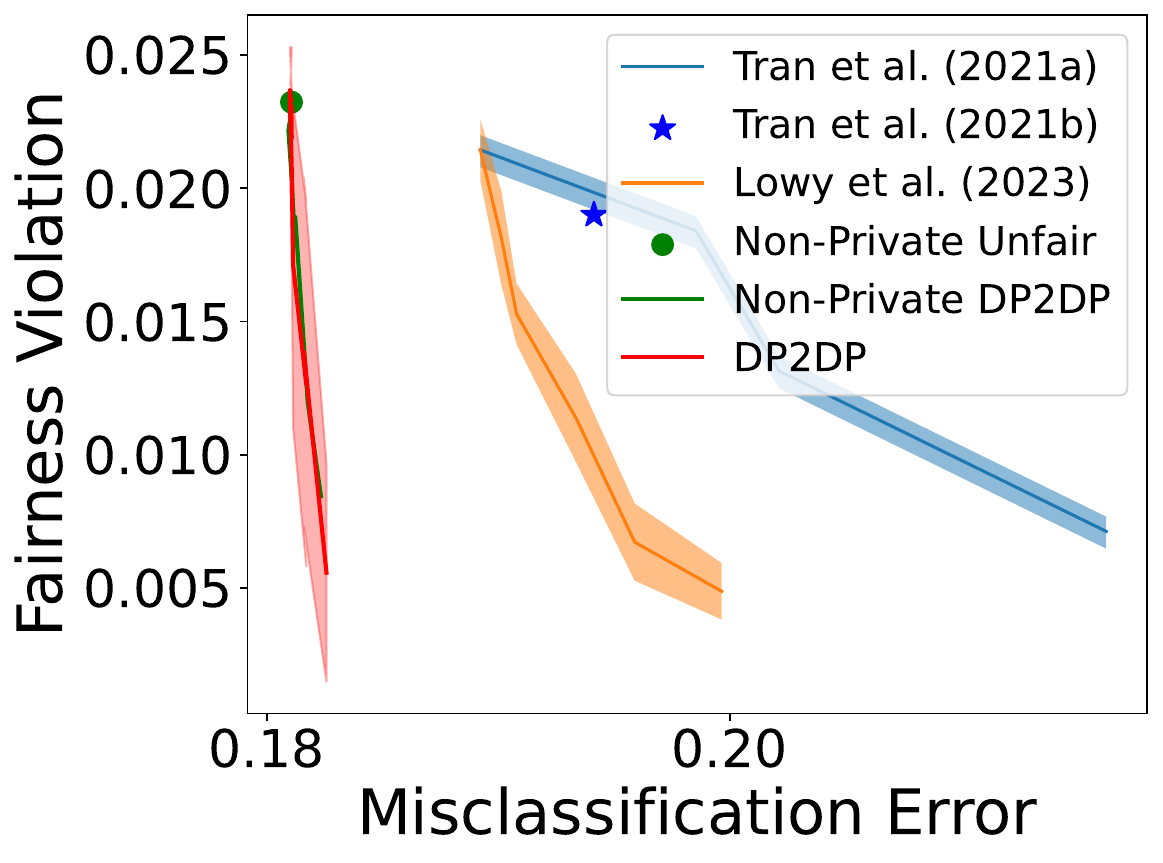}
                    \label{fig:diag_credit_eps1}
                \end{subfigure} 
                 \vspace{-15pt} 
                \caption{Comparison of our method (in terms of fairness/misclassification) with previous work on the Default-CCC dataset under $(\varepsilon, \delta)$-differential privacy with $\delta = 10^{-5}$. The left panel shows  $\varepsilon = 0.5$, and the right panel for $\varepsilon = 1.0$.}
                \label{fig:diagram_default}
            \end{figure}
        \subsubsection{Results on Parkinson's dataset}
            We report in Figure~\ref{fig:diagram_parkinson} additional results on the \textsc{Parkinson} dataset under the same privacy settings. Similar trends are observed: our method outperforms prior baselines in terms of fairness while maintaining accuracy, particularly in the low-privacy regime. These results further confirm the robustness of our approach across heterogeneous datasets with different statistical structures.

            \begin{figure}[ht]
                \centering
                \includegraphics[width=0.5\textwidth]{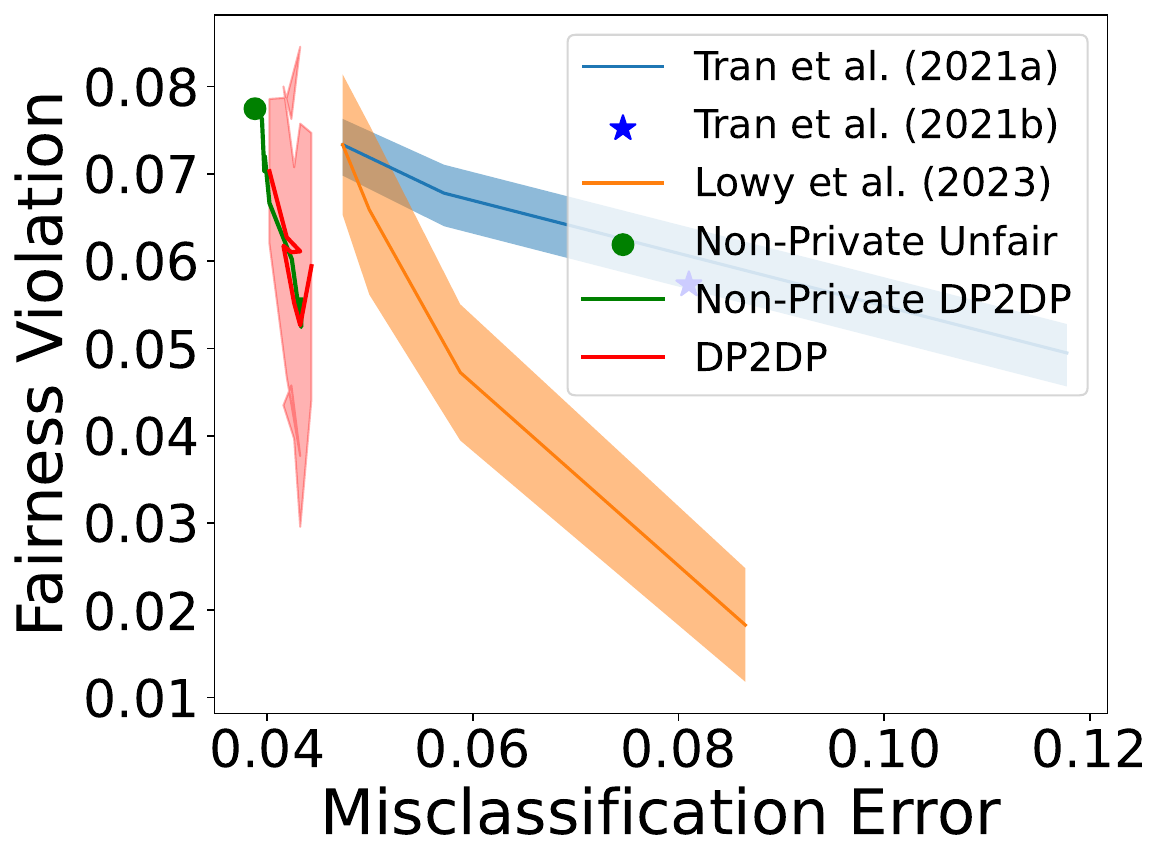}
                \caption{Comparison of our method (in terms of fairness/misclassification) with previous work on the Parkinson dataset under $(\varepsilon, \delta)$-differential privacy with $\delta = 10^{-5}$.} 
                \label{fig:diagram_parkinson}
            \end{figure}

\end{document}